\documentclass[journal]{IEEEtran}

\usepackage{cite}          % needed to automatically sort the references
\usepackage{graphicx}      % for including graphics
\usepackage{cite}
\usepackage{subfig}
\usepackage{amsmath,amssymb,amsfonts}
\usepackage{array}
\usepackage{url}
\usepackage{booktabs}
\usepackage{subfig}
\usepackage{multirow}
\usepackage{epstopdf}
\usepackage{caption}
\usepackage{color}
\usepackage{comment}
\usepackage{xspace}
\usepackage{bm}
\usepackage{multirow}
\usepackage{inputenc}
\usepackage{enumerate} 
\usepackage{textcomp}
\usepackage{listings}
\usepackage{array}
\usepackage{tikz}
\usetikzlibrary{shapes.geometric, arrows}
\usepackage{dirtytalk}
\usepackage[switch,mathlines]{lineno}
\usepackage{soul}
\usepackage{xfrac}
\usepackage{bbm}
\usepackage{breqn}
\usepackage{graphicx}
\usepackage{tikz}
\usepackage{stfloats} 
\usetikzlibrary{
  positioning,   % for “below=… of …” syntax
  arrows,        % arrow tips like -latex
  shapes,        % block/node shapes
  fit            % dashed “loop” container
}

\usepackage{subfig}
\usepackage{float}
\usepackage{amsthm}
\newtheorem{theorem}{Theorem}
\newtheorem{lemma}{Lemma}

\theoremstyle{remark}
\newtheorem{remark}{Remark}

\usepackage{algorithm,algpseudocode}
\algnewcommand{\Inputs}[1]{%
  \State \textbf{Inputs:}
  \Statex \hspace*{\algorithmicindent}\parbox[t]{.8\linewidth}{\raggedright #1}
}
\algnewcommand{\Initialize}[1]{%
  \State \textbf{Initialize:}
  \Statex \hspace*{\algorithmicindent}\parbox[t]{.8\linewidth}{\raggedright #1}
}
\usepackage{algpseudocode}

\usepackage{cancel}
\usepackage{cite}
\usepackage{bm}
\usepackage{subfig}
\usepackage{float} 

\begin{document}

\title{Game-Theoretic Reinforcement Learning based Topology-Agnostic Defense for Virtual Power Plants}

\author{Pooja Algikar,~\IEEEmembership{Member,~IEEE}, Stefano Fenu, Kumarsinh Jhala,~\IEEEmembership{Senior Member,~IEEE,}
        and~Siby Plathottam,~\IEEEmembership{Member,~IEEE}%
\thanks{%Manuscript received xxxx xx, 2025; revised xxxx xx, 2025; accepted xxxx xx, 2025. 
%Date of publication xxxx xx, 2025; date of current version xxxx xx, 2025. 
This work was supported by \textit{(DOE Office of Cybersecurity, Energy Security, and Emergency Response (CESER))}}% 
\thanks{Authors are with the energy systems and infrastructure assessment division at Argonne National Laboratory}%
\thanks{Corresponding author: Pooja Algikar (e-mail: palgikar@anl.gov).}%
\thanks{Accepted manuscript of an article published in \emph{IEEE Transactions on Smart Grid}, Early Access, 8 October 2026, doi: 10.1109/TSG.2026.3741535. \textcopyright~2026 IEEE. Personal use of this material is permitted. Permission from IEEE must be obtained for all other uses, in any current or future media, including reprinting/republishing this material for advertising or promotional purposes, creating new collective works, for resale or redistribution to servers or lists, or reuse of any copyrighted component of this work in other works.}%
}

\markboth{IEEE Transactions on Smart Grid, Early Access, 2026}%
{Author \MakeLowercase{\textit{ Algikar et al.}}: Game-Theoretic Reinforcement Learning for Cybersecurity of VPPs}
\maketitle
\begin{abstract}

\textcolor{black} {
The rapid proliferation of distributed energy resources (DERs) aggregated by virtual power plants (VPPs) elevates cybersecurity risk by expanding the attack surface. A coordinated cyberattack has the potential to disrupt distribution-system, bulk-system, and market operations. Existing cyber-physical defense strategies largely assume access to grid network models and operating conditions, which may not be available to third-party aggregators. This paper proposes a novel hybrid, perturbation-based cyber-physical defense for VPP-operated DERs that blends game-theoretic planning with model-free reinforcement learning (RL) under a Bayesian belief-update framework. Topology agnosticism is defined from the VPP information and implementation perspective: the VPP does not require access to the feeder topology, line parameters, switching configuration, or full system state. This work advances the state of the art through three contributions: (1) a multi-output Gaussian process with Hamiltonian Monte Carlo for discrete attacker-type inference, probability of detection, and confidence of detection for modeling the belief distribution; (2) moment-matched Gaussian propagation of posterior beliefs across episodes; and (3) an RL policy that adaptively integrates Stackelberg leader (VPP) actions with exploration. Furthermore, we establish equilibrium existence and derive an episode-level closed-loop utility bound connecting attacker-belief uncertainty with GT--RL action alignment. Evaluations on IEEE 13- and 123-Bus feeders with DER portfolios demonstrate effective detection under attacker-model uncertainty and evaluated feeder-reconfiguration conditions, while requiring no explicit feeder-topology knowledge. Overall, the results indicate that combining bi-level game-theoretic planning, nonparametric Bayesian inference, and adaptive RL is a promising direction for VPP cybersecurity.
}

%We propose a hybrid defense algorithm for distributed energy resources and virtual power plants (VPPs) that blends exact game-theoretic planning with model-free reinforcement learning (RL) under a Bayesian belief-update framework. We highlight three contributions: (1) a multi-output Gaussian process (GP) with Hamiltonian Monte Carlo (HMC) for discrete attacker-type inference, probability of detection, and confidence of detection; (2) moment-matched Gaussian propagation of posterior beliefs across episodes; (3) an RL policy that adaptively integrates exact Stackelberg leader (VPP) actions with exploration. Furthermore, we derive the theoretical regret bounds on defender's utility due to mis-specification of attacker's profile. We demonstrate the efficacy of the algorithm in the IEEE 13- and IEEE 123-bus system. 
% {Experiments on a realistic DER/VPP set-up demonstrate faster adaptation to novel attack strategies, improved sample efficiency, and enhanced robustness compared to end-to-end learned Stackelberg and Dirichlet-based baselines. Our results underscore the value of combining exact bilevel GT solvers, nonparametric Bayesian inference, and adaptive policy blending for cyber-physical security of virtual power plants.}
\end{abstract}
\vspace{-5mm}
\section*{Nomenclature}
\addcontentsline{toc}{section}{Nomenclature}
\begin{IEEEdescription}[\IEEEusemathlabelsep\IEEEsetlabelwidth{$V_1,V_2,V_3$}]
\item[$\bm{\theta}$] Attacker's profile
\item[$a$] Aggressiveness parameter in $\bm{\theta}$
\item[$s$] Stealth parameter in the $\bm{\theta}$
\item[$\bm{\theta}_{\text{type}}$] Type of the attack parameter in $\bm{\theta}$ set to $\{1 ,2, 3\}$ for the FDI, MitM, and no attack, respectively. 
\item[$\beta(\bm{\theta})$] belief distribution of the defender over attacker's profile 
\item[$\bm{a}_{A}^{d}(t)$] Attacker's action space for the $d$-th DER 
at time $t$
\item[$V_{A}^{d}$] Malicious data injected by the attacker in the case of FDI attack
\item[$P_{A}^{d}$] Change in the set point of the active power of the $d$-th DER induced by the attacker in the case of MitM attack
\item[$ Q_{A}^{d}$] Change in the set point of the reactive power and reactive power of the $d$-th DER induced by the attacker in the case of MitM attack
\item[$t_{A}^{d}$] timestamp of the attack values for the $d$-th DER.
\item[$\Delta V_{P}^{d}$] VP signal introduced for the $d$-th DER as a result of defender's perturbations. 
\item[$\Delta P_{D}^{d}, \Delta Q_{D}^{d}$] Magnitude of the perturbations signals produced in the active and reactive power, respectively, by the defender for the $d$-th DER. 
% \item[$t_{D}^{d}$] Timestamp of the perturbation signals by the defender for the $d-$th DER. 
\item[$V_{o}^{d}$] Voltage observations for the $d$-th DER 
\item[$\bm{a}_{RL}^{d}$] Action space of the RL agent
% \item[$\Delta P_{RL}^{d}$] Active power perturbation signals obtained by RL agent for $d-$th DER. 
% \item[$\Delta Q_{RL}^{d}$] Reactive power perturbation signals obtained by RL agent for $d-$th DER.
\item[$t_{RL}^{d}$] Timestamps of the perturbation signals obtained by RL agent for $d$-th DER. 
\item[$\bm{z}(t)$] State vector of the RL agent at time $t$ 
% \item[$r_t(\phi)$] reward function of the RL agent with internal parameters $\phi$
\item[$A$] advantage of the reward
% \item[$\varepsilon$] clipping parameter in the proximal policy update likelihood
% \item[$U_{D}$] Utility function of the defender in the GT module 
% \item[$U_{A}$] Utility function of the attacker in the GT module
% \item[$I$] impact of the attack/ defense mechanism
% \item[$p_{d}$] probability of detection in the attack or defender's utility 
% \item[$C_{\mathrm{det}}$] instantaneous confidence of detection (from the detector)
% \item[$P_{\mathrm{det}}$] Probability of detection (from the detector)
% \item[$\theta_{\mathrm{det}}$] Label of attack (from the detector)
% \item[$\tau$] index of the episode in GT-RL algorithm
% \item[$T$] Duration of an episode in GT-RL algorithm  
\item[$\bm{o}$] observations on confidence of detection $C_{D}$, probability of detection $P_{D}$, attack label $\theta_{type}$ obtained from the detector module
\item[$k(\cdot,\cdot)$] kernel function   

\item[$l$] characteristic length parameter in the kernel function  
\item[$x_{p},x_{q}$] coefficients for the regression for estimating the VP signal
\item[$e(t)$] i.i.d. noise of the regression model for the VP signal 
\item[$\rho^{d}$] vigilance factor of the $d$-th DER 
% \item[$\gamma$] threshold of detection set for the detector
\item[$P_{\mathrm{FA}}$] probability of false alarm set for the detector
% \item[$\mathcal{T}$] test statistic 
\end{IEEEdescription}

%\section*{Notations}
%Throughout the paper, we used italic bold letters to represent a vector and italic unbold letters to represent a variable.
%$|{A}|$ denotes the magnitude of ${A}$.

\section{Introduction} Grid modernization via distributed energy resources (DERs), microgrids, third‐party aggregators, and virtual power plants (VPPs) enhances flexibility but expands the cybersecurity attack surface. Consumer-owned DERs managed by VPPs form complex communication and control pathways. Many of these assets are small and fall outside traditional utility standards, such as the North American Electric Reliability Corporation Critical Infrastructure Protection (NERC CIP) requirements; yet, in aggregate, they can pose significant risks to the bulk power system. As VPP‐operated DERs proliferate, breaches risk VPP operations, distribution stability, and market integrity \cite{CMG+19}. The National Electric Sector Cybersecurity Organization Resource (NESCOR) has cataloged a broad spectrum of cyber threats targeting DER ecosystems \cite{doe2022dercyber,roy2024defense}. These vulnerabilities span from aggregators and communication links to control interfaces, creating multiple entry points for adversaries. Exploiting these weak spots, attackers can launch \emph{false data injection} (FDI) and \emph{man-in-the-middle} (MitM) attacks that lead to financial loss, degraded power quality, equipment damage, or even system instability \cite{YLF20,SAC+23}. While FDI attacks compromise the integrity of sensor measurements and control commands, MitM attacks intercept and manipulate the communication between VPPs and DERs, undermining trust and coordination across the distributed grid.

Modern smart grids now encompass a diverse ecosystem of stakeholders beyond traditional distribution utilities, including DER aggregators, demand response operators, electric vehicle (EV) charging coordinators, and VPPs. 
{
In broader industrial control settings, context-aware fusion of heterogeneous physical and communication-layer observations has also been investigated to improve semantic understanding of system behavior~\cite{hetu2025casi}. Such approaches demonstrate the value of combining complementary cyber and physical evidence; however, they generally assume access to sensing modalities and communication information that may not be available to a third-party VPP operator.}
\textcolor{black} {This growing heterogeneity introduces a wide range of devices, operational constraints, and irregular data streams, creating security challenges that utility-centric defenses alone cannot address. VPPs, in particular, operate with limited visibility into the underlying distribution-grid topology and state, which weakens conventional model-based defenses~\cite{NISTIR7628}. We therefore propose a topology-agnostic perturbation-based defense, where topology agnosticism is defined from the VPP information and implementation perspective: the method does not receive the feeder topology or solve an online power-flow model~\cite{yin2024advancing,shrestha2024anomaly}.}

Attackers are often \emph{strategic}, adapting their behavior to exploit predictable defenses. Static RL policies, once learned, can therefore be evaded through subtle or intermittent attack patterns \cite{saber2025model}, highlighting the need for \emph{adaptive} defense mechanisms. Combining game theory with RL offers such adaptivity: game‐theoretic formulations capture the dynamic interplay between attacker and defender, optimizing monitoring and resource allocation under uncertainty, while RL equips the defender to autonomously learn effective detection and mitigation strategies in complex, partially known cyber‐physical systems.
\vspace{-3mm}
\subsection{Related Work} 
{
Beyond conventional cyberattacks on grid measurements and control channels, the security of the learning models used within cyber-physical defense systems has also
received increasing attention. Recent work has examined physics-constrained robustness and data-security guarantees for machine-learning models deployed in low-carbon smart
grids~\cite{zhang2023data}, while broader surveys have documented adversarial
vulnerabilities of ML-based applications across IoT-enabled generation, transmission,
distribution, and consumption systems~\cite{zhang2024vulnerability}. These studies
highlight that learned detectors and controllers may themselves be exposed to
adversarially perturbed measurements, motivating defense mechanisms that combine
physical consistency checks with adaptive decision-making.}

There is growing interest in smart‐grid cybersecurity using game theory and RL \cite{roberts2020deep,hasan2020game,lakshminarayana2021moving,hausken2009minmax,yao2025game}. Hasan et al. \cite{hasan2020game} cast static/dynamic interactions between the attacker and the defender as two‐player budget‐constrained games over substations, deriving polynomial‐time algorithms for worst‐case attacks and optimal protection under uniform costs; their combinatorial approach allocates protection budgets but lacks online detection, attacker‐type uncertainty, and continuous control. Roberts et al. \cite{roberts2020deep} and Selim et al. \cite{selim2023adaptive} train model‐free deep RL agents that map grid states to DER setpoint adjustments to mitigate voltage violations and load shedding, treating attacks as stochastic disturbances without explicit attacker models or detectors. Lakshminarayana et al. \cite{lakshminarayana2021moving} propose moving‐target defense via D‐FACTS impedance perturbations, computing a Nash equilibrium for line perturbations; the static/periodic game focuses on network randomization rather than detection or adaptive learning.

VPP‐specific cyber‐defense remains scarce, especially given a VPP’s limited visibility beyond its fleet. 
% Our method embeds deep RL within the defender’s control loop: the RL agent learns continuous, real‐time perturbations inside a Bayesian Stackelberg action envelope. 
Ni and Paul \cite{ni2019multistage} study a multistage Markov game on transmission lines with off‐policy Q‐learning to minimize cumulative damage. Their approach differs from ours in several key aspects: (i) it models simultaneous attacker and defender actions over a fixed set of transmission lines with explicit state tracking of breached or secured assets; (ii) it assumes all attacks are observable, without accounting for stealthy perturbations or alarm thresholds; and (iii) its reinforcement learning component approximates the equilibrium offline, without belief updates or online adaptation beyond the pre-trained policy. Yao et al. \cite{yao2025game} compute time‐varying defense via Bayesian Attack Graphs that score plant vulnerabilities and a two‐player stochastic game for resource allocation; a bespoke RL algorithm learns the Nash equilibrium for which plants to protect at each stage.

% Earlier perturbation-based detection methods established the foundation for using DER inverters to generate small, intentional voltage perturbations for verifying measurement authenticity. \cite{JPC+21,JPN21} formulated this concept analytically using voltage sensitivity analysis to identify optimal DER locations and magnitudes for injecting low-magnitude perturbations. The detection method was rule-based: the presence or absence of the known perturbation pattern in sensor data indicated whether a measurement was authentic. \cite{FJS+25} extended this idea by introducing an RL agent to automate and adapt the perturbation process. Instead of precomputing DER locations analytically, the RL agent learned which DERs to perturb and by how much, optimizing detection effectiveness over time. However, this RL formulation was run in a fixed environment without explicitly modeling uncertainty about attacker's profile. 
% Early perturbation-based methods used DER inverters to inject small, intentional voltage signals and verify measurement authenticity via rule-based detection. Analytical formulations in earlier methods \cite{JPC+21,JPN21} identified optimal DER locations and perturbation magnitudes using voltage sensitivities, with detection based on whether the known pattern appeared in sensor data. This idea was later extended in \cite{FJS+25}, where an RL agent learned when and how to perturb DERs, removing the need for precomputed analytical placement. However, this RL approach assumed a fixed environment and did not model uncertainty in the attacker’s behavior. 

{
The proposed perturbation detector follows an active challenge–response principle: the VPP injects a small secret perturbation into selected DER setpoints and tests whether the induced voltage footprint appears in observer measurements. This idea is closely related to watermarking-based attack detection in control systems \cite{khazraei2017new,porter2020detecting,bessa2022dual,zhang2023watermarking}, where intentional excitation (known only to the defender) increases detectability of stealthy false-data attacks, and to moving-target defense in power grids \cite{piaquadio2025pulse,cavraro2019inverter} that uses controlled perturbations to reduce attacker stealth. In distribution systems settings, prior work has used voltage-sensitivity models to design such probing perturbations and rule-based detectors \cite{JPC+21,JPN21}, and has recently explored RL-driven perturbation selection \cite{FJS+25}.  These prior approaches employ RL, but they typically assume (i) full access to the distribution network topology \cite{cavraro2019inverter} and (ii) ground-truth detection labels to shape the reward and improve the policy. In VPP settings, neither assumption reliably holds: a VPP often observes only the DERs in its own fleet rather than the full feeder topology, and distribution topology can change frequently with ongoing electrification (e.g., new DER interconnections, switching operations, and reconfiguration). Our contribution differs by integrating attacker-type uncertainty via a Bayesian Stackelberg game layer that shapes the RL perturbation policy and updates beliefs online from detector outputs.}

In contrast, our framework takes a model-agnostic view. {Our contribution differs by integrating attacker-type uncertainty via a Bayesian Stackelberg game layer that shapes the RL perturbation policy and updates beliefs online from detector outputs.} We introduce small, secret perturbations into the power outputs of selected DERs and observe how these perturbations propagate through voltage measurements of other DERs. Under normal operation, these voltage measurements follow predictable patterns; deviations from expected behavior reveal anomalies that may indicate cyberattacks. To enhance detection and adaptivity, we couple this perturbation-based mechanism with a hybrid Bayesian Stackelberg–reinforcement learning (GT-RL) approach. 
The game-theoretic layer models attacker uncertainty and strategic behavior, while the RL agent learns real-time perturbation strategies guided by game-theoretic insights into the reward function. Crucially, this design enables a VPP to implement effective cybersecurity using only selective voltage measurements from observer DERs, without requiring complete knowledge of network topology or full system states.

% Earlier analytical work used voltage sensitivities to choose DER perturbations, their locations, and relied on rule-based detection \cite{JPC+21,JPN21}. This was later extended with an RL agent that learned adaptive perturbation strategies \cite{FJS+25}, although it operated in a fixed environment without modeling uncertainty in attacker behavior.
\subsection{Contributions}
Building on the prior perturbation‐based detection \cite{FJS+25,JPC+21,JPN21}, we extend to a Bayesian Stackelberg game models the interactions of the attacker and the defender (VPP) and shapes the reward function of the RL agent, which generates perturbation signals for the real-time defense.  
The defender (as leader) maintains Gaussian‐process belief over attacker's profile parameters and updates them using perturbation detector outputs. 
Our main contributions are: 
\begin{itemize}
    \item A topology‐agnostic, perturbation‐based defense mechanism for VPP‐operated DERs.
    \item  A hybrid framework combining Bayesian Stackelberg game theory (to capture attacker–defender interactions under attacker profile uncertainty) with RL agent for adaptive, real‐time strategy optimization.
    \item  A mathematical formulation of the utilities, action spaces, and learning objectives for the integrated game-theoretic and RL components, together with a theoretical proof of equilibrium existence and an episode-level closed-loop utility bound connecting attacker-belief uncertainty with GT--RL action alignment.
\end{itemize}
{
The proposed detector and GT-RL policy are \emph{topology-agnostic} in a sense that they do not require an explicit feeder model (e.g., line parameters, admittance matrix, or switching configuration) nor do they require solving topology-dependent power-flow computations online. 
% Instead, the detection module relies on an \emph{empirical} local linear response model that maps the defender’s injected perturbations \((\Delta P,\Delta Q)\) to observed voltage deviations \(\Delta V\). Specifically, the coefficients \((x_p,x_q)\) are identified from data using weighted least squares (WLS) under trusted/normal operation, and detection is performed by testing whether the expected perturbation footprint remains consistent in the observation window.
}
% The remainder of the paper is organized as follows. The proposed method is developed through Section~\ref{sec3,sec4,sec5,sec6,sec7}. In Section~\ref{sec8} we illustrate the proposed framework on the IEEE 13-Bus and 123-Bus system and present a detailed analysis of the impact of game theoretic layer to the RL agent. Section \ref{sec9} concludes the paper and outlines future work.

% (i) A topology‐agnostic, perturbation‐based defense mechanism for VPP‐operated DERs. (ii) A hybrid framework combining Bayesian Stackelberg game theory (to capture attacker–defender interactions under attacker profile uncertainty) with RL agent for adaptive, real‐time strategy optimization. (iii) A mathematical formulation of utilities, action spaces, and learning objectives for the integrated game‐theoretic and RL components. (iv) Theoretical bound on regret defined on defender's utility due to mis-specification of attacker's profile.

\begin{figure}
    \centering
\includegraphics[width=1.02\linewidth]{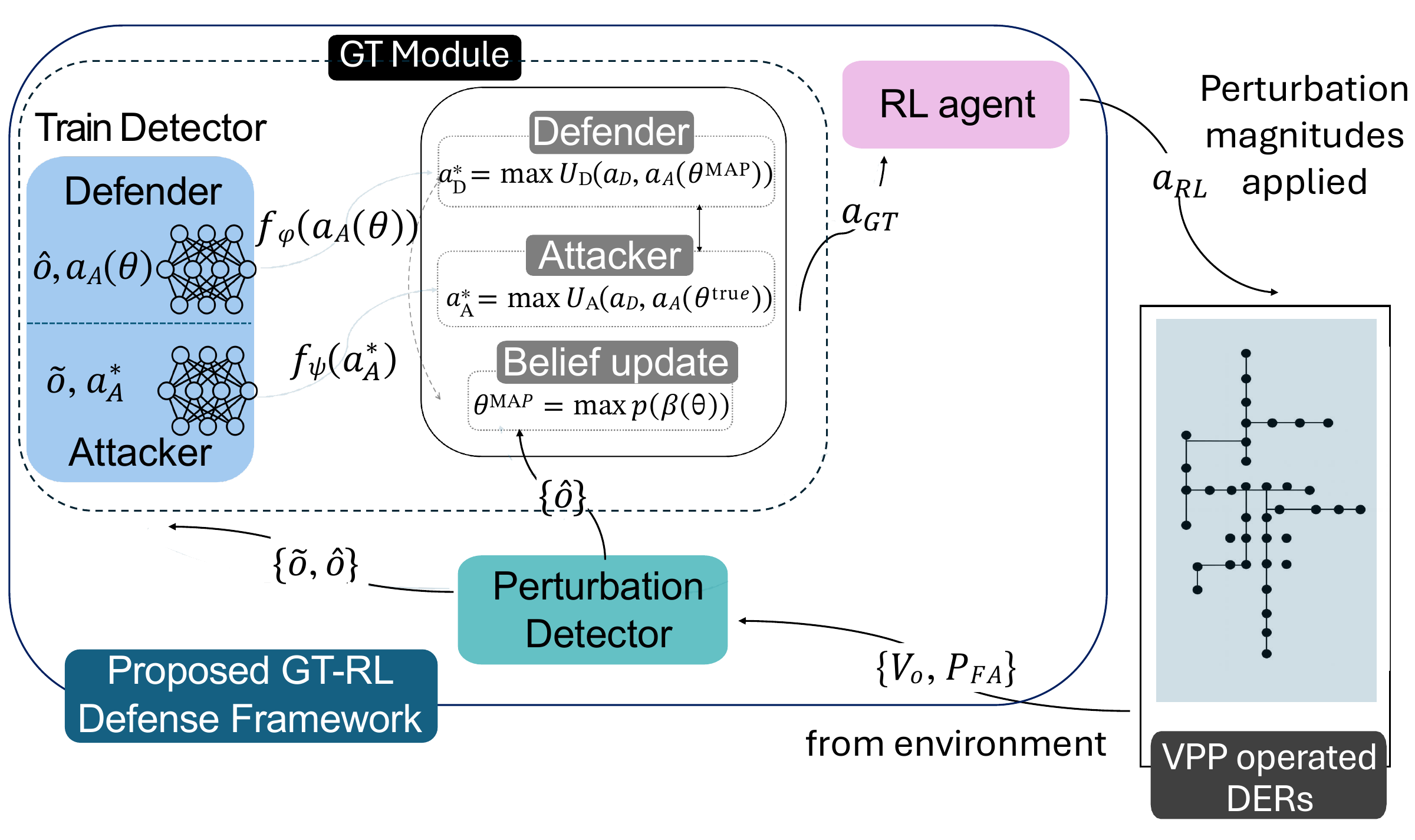}
    \caption{The proposed GT-RL method optimizes the perturbation magnitudes applied to the physical DER environment, while the environment returns voltage observations \(V_o\) and the prescribed false-alarm setting \(P_{FA}\) to the perturbation detector for watermark-response evaluation.}
    \label{depict}
\end{figure}

\section{System Model}

{
We consider a VPP that aggregates a fleet of DERs connected to an underlying electric distribution network. In this work, the \emph{defender} refers to the VPP operator that seeks to detect cyber-physical attacks affecting DER measurements, DER setpoints, or DER communication channels. Let $\mathcal{D}$ denote the set of DERs operated by the VPP. For each DER $d \in \mathcal{D}$, the VPP has access to DER-level quantities such as active power $P^d(t)$, reactive power $Q^d(t)$, and local voltage magnitude $V_o^d(t)$ measured at the DER point of interconnection. These quantities are assumed to be available through the DER controller, inverter interface, or aggregator communication layer.
}

{
A central assumption of this work is that the VPP does not have access to the full distribution-system model including the feeder topology, line impedances, admittance matrix, switching configuration, or complete network operating state. Such information is typically owned by the distribution utility and may not be shared with third-party aggregators. Therefore, conventional cyber-physical detection methods that require a full power-flow model, state estimator, or topology-dependent sensitivity matrix are not directly applicable from the VPP operator's perspective. The proposed method is designed for this limited-information setting.
}

{
The perturbation-based detection principle builds on our prior work on
DER-enabled attack detection~\cite{JPN21,JPC+21} and its extension to a topology-unobserved VPP setting~\cite{FJS+25}. Accordingly, the use of small physical perturbations as a detection signature is not itself the primary novelty. The main contribution is the integration of Bayesian Stackelberg reasoning, belief-based attacker modeling, and reinforcement learning for adaptive perturbation design under attacker uncertainty.
}

{
The VPP selects actor DERs
\(\mathcal{D}_{\mathrm{act}}\subseteq\mathcal{D}\), whose active- and
reactive-power setpoints can be perturbed, and observer DERs
\(\mathcal{D}_{\mathrm{obs}}\subseteq\mathcal{D}\), whose voltage
measurements are monitored. A DER may serve as an actor, an observer, or both.
For \(d\in\mathcal{D}_{\mathrm{act}}\), the defender injects small,
time-stamped perturbations \(\Delta P_D^d(t)\) and \(\Delta Q_D^d(t)\).
The perturbation sequence is known to the VPP and assumed unknown to the
attacker. Its magnitude is constrained by DER capability, power-quality, and
operational requirements, and is kept sufficiently small to support a local
linear approximation around the current operating point.
}

{
For an observer DER \(o\in\mathcal{D}_{\mathrm{obs}}\), the resulting voltage
response is approximated as
}
\begin{equation}
{
\Delta V_o(t)
\approx
x_{p,o}\Delta P_D(t)
+
x_{q,o}\Delta Q_D(t)
+
e_o(t),
}
\label{eq:local_voltage_response_system}
\end{equation}
{
where \(\Delta V_o(t)\) is the observed voltage deviation,
\(\Delta P_D(t)\) and \(\Delta Q_D(t)\) denote the aggregate or selected
actor-DER perturbations, \(x_{p,o}\) and \(x_{q,o}\) are local active- and
reactive-power-to-voltage sensitivity coefficients, and \(e_o(t)\) denotes
voltage-fluctuation noise caused by load variation, stochastic DER output,
measurement noise, and residual modeling error. The coefficients
\(x_{p,o}\) and \(x_{q,o}\) are topology-implicit because they depend on the
feeder configuration, line parameters, and operating state. However, they are
estimated from perturbation-response measurements rather than computed from an
explicit feeder model. Thus, the physics remains topology-dependent, while
the VPP implementation does not require topology information. For multiple
actor DERs in the small-perturbation regime, the observer-voltage response is
approximated by superposition of their individual active- and reactive-power
contributions~\cite{munikoti2021probabilistic}. This approximation is local to
the current operating point and is not assumed to remain valid under
arbitrarily large or abrupt system changes.
}

{
The defender-induced voltage response acts as a secret physical signature,
analogous to a dynamic watermark. Under benign operation, the expected voltage
perturbation (VP) signal should be present in the observer measurements.
Measurement tampering, setpoint manipulation, packet delay, or communication
loss may distort, suppress, delay, or remove this signal. Attack detection is
therefore formulated as determining whether the known VP signal is present in
noisy voltage observations.
}

{
The perturbation detector performs only this statistical signal-detection
task. Using the known perturbation sequence, measured voltage response, and
estimated sensitivities, it produces a binary decision, detection probability,
and confidence measure. These outputs indicate whether the observed response
is consistent with the expected VP signal and provide evidence to the
higher-level belief-update and decision-making modules; the detector itself
does not directly infer the strategic attacker profile.
}

{
The attacker is modeled as an intelligent adversary who may compromise the cyber-physical operation of one or more DERs. The attacker may manipulate measurements, alter control commands or setpoints, or disrupt the availability of communication packets. In this work, representative attack classes include FDI, MitM, and DoS attacks, which are commonly studied attack categories in smart-grid and DER cybersecurity~\cite{deng2018false,huseinovic2020survey,raja2024mitigation,wlazlo2021mitm}. The defender does not know the attack type, timing, magnitude, or stealth level in advance. These unknown attack characteristics are represented through an attacker profile $\bm{\theta}$, which is inferred over time from the evidence produced by the detector.
}

{
It is important to distinguish system uncertainty from attacker uncertainty. System uncertainty refers to benign variations in the voltage response caused by load changes, stochastic DER output, measurement noise, and small operating-point fluctuations. This uncertainty appears as the residual term in the voltage observation model and affects the statistical reliability of the detector. Attacker uncertainty refers to the defender's lack of knowledge about the attacker's type, intent, timing, aggressiveness, and stealth. This uncertainty is handled by the game-theoretic and belief-update components. Thus, the proposed framework does not treat natural system variability as an attack feature; instead, it separates signal-detection uncertainty from strategic attacker-profile uncertainty.
}

{
The RL agent is introduced because the VPP must decide how to choose perturbation actions under partial observability, stochastic operating conditions, and limited network information. If the VPP knew the full network model and operating state, it could compute topology-dependent sensitivity information and design perturbations analytically, as in prior distribution-utility-oriented perturbation-based detection methods~\cite{JPN21,JPC+21}. In the third-party VPP setting, however, the defender must learn from observed responses. The RL agent uses recent system observations and detector feedback to adapt the magnitude, timing, and allocation of perturbations across actor DERs. Its objective is to improve detectability while limiting unnecessary perturbation cost and avoiding operationally excessive setpoint changes. This extends our prior RL-based VPP perturbation-defense work~\cite{FJS+25} by incorporating strategic attacker uncertainty and game-theoretic belief guidance.
}

{
The game-theoretic module is introduced because the attacker is not merely a random disturbance; it may adapt strategically to the defender's actions. The defender therefore maintains a belief distribution over possible attacker profiles and uses a Bayesian Stackelberg game to reason about attacker--defender interactions. In this formulation, the VPP acts as the leader that selects a defense strategy, while the attacker is modeled as a follower that responds according to its profile and objective. The game-theoretic module provides a strategic candidate perturbation action or guidance signal to the RL agent based on the current belief over attacker behavior. The RL agent then adapts this guidance using real-time observations and detector feedback.
}

{
The overall information flow is as follows. First, the VPP operator, acting as the defender, selects small perturbations for actor DERs. Second, the resulting voltage responses are measured at observer DERs. Third, the detector tests whether the expected perturbation footprint is present in the voltage observations. Fourth, the detector outputs are used to update the defender's belief over attacker profiles. Fifth, the Bayesian Stackelberg module uses the updated belief to generate strategic defense guidance. Finally, the RL agent uses this guidance together with recent system observations to choose future perturbation actions. This closed-loop structure allows the VPP to perform adaptive cyber-physical defense using only DER-level measurements and controls, without requiring explicit access to the distribution-network topology or full system state.
}
\section{Perturbation Detector Design}\label{sec3}
The VPP employs a perturbation‑based defense, analogous to the application of a dynamic watermark whose absence would indicate signal tampering. 
We design the statistical perturbation-based detector. A target log-likelihood ratio is set as a threshold to turn the resulting confidence measure into a boolean flag indicating a potential attack.
% and we show how modifying this threshold using game-theoretic fine tuning dramatically increases system performance. 
% More formally: 

The actor nodes $d_{a}\in\mathcal{D}_{act}$ apply small, time-stamped power injections by adjusting their active and reactive power \(\{\Delta P_D^{d_{a}}(t), \Delta Q_D^{d_{a}}(t)\}\). The resulting voltage responses \(V_o\) are then monitored by the observer DERs \( d \in \mathcal{D}_{obs} \) to detect potential tampering.

These defender’s perturbations induce a voltage deviation $\Delta V_P(t)$, called the differential voltage perturbation signal, at all the observer DERs, proportional to the injected perturbations:  
\begin{equation}  
\small
    \Delta V_P^d(t) = x_p \Delta P_D^{d_{a}}(t) + x_q \Delta Q_D^{d_{a}}(t),\; \forall d\in\mathcal{D}_{obs},  \forall t \in T_{o}.  
\end{equation}  
The detector then identifies whether the $V_{p}$ signal is present in the voltage observations $V_{o}^d$ during the observation period $T_{o}$.
\begin{equation}
\small
    \Delta V_{o}^{d}(t)=x_p \Delta P_{D}^{{d}_{a}}(t) + x_q \Delta Q_{D}^{{d}_{a}}(t) + e(t), \forall d \in D_{obs},
\end{equation}
where $e(t)\sim \mathcal{N}(0,\sigma^{2})$.

{
The sensitivity coefficients $x_p$ and $x_q$ represent the nominal active and reactive power-to-voltage sensitivities of the voltage observations.}
At each iteration of the defender's mechanism, the detector determines the presence of $\Delta V_{\text{P}}$ signal by gathering observations $\bm{V}_{o}=[V_{o}(1), \hdots, V_{o}(T_{o})]$ from the period of observation $T_{o}$. 
{
The noise term models system uncertainty accounting time varying load changes, stochastic
DER output, measurement noise, and small operating-point
fluctuations.}
The detector operates on  hypothesis-based signal detection, where: 
\[
\small
\begin{cases}
\mathcal{H}_0: & \text{Observation contains only noise (No sequence detected)}, \\
\mathcal{H}_1: & \text{Observation contains $V_{\text{P}}$ signal (Sequence detected)}.
\end{cases}
\]
The set of observations $\Delta \bm{V}_{o}=[\Delta V_{o}(2), \hdots, \Delta V_{o}(T_{o})]$ allow the null hypothesis to be $\mathcal{H}_0: \Delta \bm{V}_{o} = \bm{e}$ and alternative hypothesis $\mathcal{H}_1:\Delta \bm{V}_{o}= x_{p} \Delta \bm{P}_{D}+ x_{q} \Delta \bm{Q}_{D}+ \bm{e}$. 
The coefficients $\bm{x}=[x_{p}, x_{q}]$ based on the regressor $\Delta \bm{S}= [\Delta \bm{P}_{D}, \Delta \bm{Q}_{D}]$ are estimated as 
\begin{equation}\label{wls}
    \hat{\bm{x}}= (\Delta \bm{S}^{\top}\Delta \bm{S})^{-1}\Delta \bm{S}^{\top}\Delta \bm{V}_{o}.
\end{equation}
The variance of the error i.e. the variance of the residual $\bm{r}=  \Delta  \bm{V}_{o} -  \hat{\bm{x}}^{\top}\Delta \bm{S}$ can be estimated using sample variance as:
\begin{equation}
    \hat{\sigma}^{2}= \frac{1}{T_{o}-1} \sum_{t=2}^{T_{o}} r(t)^{2}.
\end{equation}

% \subsection{Decision rule of the detector}
The optimal detector under the Neyman–Pearson framework decides on the hypothesis based the following decision rule for the log likelihood ratio $\mathcal{L}(\Delta \bm{V}_{o})$:  
\begin{equation}
    \mathcal{L}(\Delta \bm{V}_{o}) 
    \;\underset{\mathcal{H}_{0}}{\overset{\mathcal{H}_{1}}{\gtrless}}\; 
    \gamma.
\end{equation}
% \begin{figure}
%     \centering
%     \includegraphics[width=0.8\linewidth]{Figures/detector_performance_from_belief_update_mitm.png}
%     \caption{Probability of detection  $P_D$ for a MitM attack as defender updates its belief $\beta(\theta_{\text{type}})$ in the IEEE 13-Bus system.}
%     \label{fig:detection_over_belief}
% \end{figure} 
The likelihood ratio is the fraction of the probability of distribution of $\Delta \bm{V}_{o}$, $p(\Delta \bm{V}_{o})$, under $\mathcal{H}_{1}$ and $\mathcal{H}_{0}$:
    \begin{align}
    \small
        \mathcal{L}(\Delta \bm{V}_{o}) \notag
        &= \prod_{t=2}^{T_{o}} \frac{p(\Delta V_{o}(t)|\mathcal{H}_1)}{p(\Delta V_{o}(t)|\mathcal{H}_0)} = \prod_{t=2}^{T_{o}} 
 \frac{p(\Delta  V_{o}(t)- \Delta 
 \hat{V}_{o}(t))}{p(\Delta V_{o}(t))}\\
        &= \prod_{t=2}^{T_{o}}  \exp\left( \frac{1}{2\sigma^2}(2\Delta V_{o}(t) \Delta  \hat{V}_{o}(t)- \Delta  \hat{V}_{o}(t)^{2}) \right).   
    \end{align}
To ensure the decision rule is based on a monotonic (strictly increasing or decreasing) function of the likelihood ratio, we take log 
\begin{equation}
     \log( \mathcal{L}(\Delta \bm{V}_{o}))=  \frac{1}{2\sigma^2} \sum_{t=2}^{T_{o}} (2 \Delta  V_{o}(t) \Delta  \hat{V}_{o}(t)-\Delta \hat{V}_{o}(t)^{2}).  
\end{equation}
Incorporating $2\sigma^{2}$ and $\hat{V}_{o}^{2}$ into the threshold $\gamma$, the detector becomes a pure correlator detector where the threshold of detection $\gamma'=\sigma^2\log(\gamma)+ \frac{1}{2}\hat{V}_{o}^{2}$, giving us the test statistics: 
\begin{equation}
    \mathcal{T}(\Delta \bm{V}_{o})= \Delta  \bm{V}_{o}^{\top} \Delta  \hat{\bm{V}}_{o}.
\end{equation}
The distribution of $\mathcal{T}(\Delta \bm{V}_{o})$ under the hypothesis $\mathcal{H}_{0}$ and $\mathcal{H}_{1}$ is
\begin{equation}\label{15}
    \mathcal{T}(\Delta \bm{V_{o}}) \sim
\begin{cases}
\small
\mathcal{N}\left(0,\, \sigma^2 \|\Delta\hat{\bm{V}_{o}}\|^2\right), & \text{under } \mathcal{H}_0 \\\\
\mathcal{N}\left(\|\Delta\hat{\bm{V}_{o}}\|^2,\, \sigma^2 \|\Delta\hat{\bm{V}_{o}}\|^2\right), & \text{under } \mathcal{H}_1.
\end{cases}
\end{equation}
{
Under the stated Gaussian simple-hypothesis model, the Neyman--Pearson
likelihood-ratio test yields the most powerful test for a prescribed
\(P_{FA}\)~\cite{neyman1933problem,steven1993fundamentals}.
The detection threshold $\gamma$ is selected based on a desired false alarm rate $P_{\mathrm{FA}}$, that is, the probability with which the detector detects the presence of $V_P$ signal when in fact the system is under attack. 
Mathematically, we use the inverse Gaussian tail function $Q^{-1}(P_{\mathrm{FA}})$ to compute the corresponding threshold value:
\begin{equation}
    \gamma' = \sigma\, Q^{-1}(P_{\text{FA}})\, \| \Delta \hat{\bm{V}}_{o} \| + \tfrac{1}{2} \| \Delta \hat{\bm{V}}_{o} \|^2.
\end{equation}} 
%2.2

{
In our formulation, $\mathcal{H}_0$ denotes the presence of an attack, whereas $\mathcal{H}_1$ corresponds to the no-attack (benign) operating condition. The detector compares the test statistic $\mathcal{T}(\Delta \bm{V}_{o})$ against the threshold $\gamma'$ and decides in favor of $\mathcal{H}_1$ if $\mathcal{T}>\gamma'$, and $\mathcal{H}_0$ otherwise. 
Under this convention, two complementary detection probabilities naturally arise. The first is the probability of correctly identifying the no-attack condition when the system is benign, i.e., deciding $\mathcal{H}_1$ under $\mathcal{H}_1$. The second is the probability of correctly detecting an attack, i.e., deciding $\mathcal{H}_0$ under $\mathcal{H}_0$. Both quantities are functions of the same test statistic distribution and threshold, but correspond to different tails of the Gaussian distribution under their respective hypotheses. We next derive closed-form expressions for these probabilities.}

\subsection{Probability of detection}
The probability of detecting the $V_P$ signal is defined as the probability that the test statistic exceeds the threshold under the hypothesis $\mathcal{H}_1$, i.e., $p(\mathcal{T}>\gamma'| \mathcal{H}_{1})$.
This also represents the probability of correctly identifying a no-attack condition when the system is indeed not under attack.
Using the distribution $\mathcal{T} \sim \mathcal{N}\left(\|\Delta\hat{\bm{V}_{o}}\|^2,\, \sigma^2 \|\Delta\hat{\bm{V}_{o}}\|^2\right)$ from \eqref{15},
we get,
\begin{equation}
\small
    P^{\text{no-attack}}_{\mathrm{det}} = p(\mathcal{T} > \gamma' \mid \mathcal{H}_1) = Q\left( \frac{\gamma' - \|\Delta \hat{\bm{V}}_o\|^2}{\sigma \|\Delta \hat{\bm{V}}_o\|} \right),
\end{equation} 
further substituting $\gamma'=\sigma\, Q^{-1}(P_{\text{FA}})\, \| \Delta \hat{\bm{V}}_{o} \| + \tfrac{1}{2} \| \Delta \hat{\bm{V}}_{o} \|^2$, we arrive at
\begin{equation}
   P^{\text{no-attack}}_{\mathrm{det}} = Q\left( Q^{-1}(P_{\text{FA}}) - \tfrac{1}{2} \cdot \frac{\|\Delta \hat{\bm{V}}_{o} \|}{\sigma} \right),
\end{equation}
Similarly, the probability of correctly detecting an attack (deciding \( \mathcal{H}_0 \) when \( \mathcal{H}_0 \) is true) is: 
\begin{equation}
\small
    P^{\text{attack}}_{\mathrm{det}} = p(\mathcal{T} < \gamma' \mid \mathcal{H}_0) = Q\left( Q^{-1}(P_{\text{FA}}) + \tfrac{1}{2} \cdot \frac{\|\Delta \hat{\bm{V}}_{o}\|}{\sigma} \right).
\end{equation}
\subsection{Posterior confidence of detection}
Posterior confidence of detection $C_{\mathrm{det}}$ is a Bayesian measure that tells us how likely the $V_P$ signal is present, given the observed data. It’s sample-specific and based on the likelihood ratio and prior beliefs. To keep our decision rule unbiased, we use non-informative priors: \( \pi_{0}=P(\mathcal H_{0}) \), \( \pi_{1}=P(\mathcal H_{1})=1-\pi_{0} =0.5\). 
The confidence of detection of the $V_{p}$ signal i.e., the confidence of detecting no attack, $C_{\mathrm{det}}^{\text{no-attack}}(\Delta \bm{V}_{o})$, is the Bayesian posterior $C_{\mathrm{det}}^{\text{no-attack}}(\Delta \bm{V}_{o})
 = p(\mathcal H_{1} \mid \Delta \bm{V}_{o})$, can be derived as:
\begin{equation}
\small
    C_{\mathrm{det}}^{\text{no-attack}}(\Delta \bm{V}_{o})
 = p(\mathcal H_{1} \mid \Delta \bm{V}_{o})
  = \frac{\pi_{1} \mathcal{L}(\Delta \bm{V}_{o})}{\pi_{0} + \pi_{1} \mathcal{L}(\Delta \bm{V}_{o})}.
% & = \frac{1}{1 + \mathcal{L}(\Delta \bm{V}_{o})}.
\end{equation}
Similarly, confidence of not detecting the $V_{p}$ signal implies the presence of attack. Therefore the confidence of detecting an attack $C_{\mathrm{det}}^{\text{attack}}$ can be given by
    $C_{\mathrm{det}}^{\text{attack}}= 1-C_{\mathrm{det}}^{\text{no-attack}}$.

\section{Stackelberg Game Setting}\label{sec4}  
To address challenges of the uncertainty of the attacker  cybersecurity  in VPP-operated DERs, the interaction between the defender (VPP operator) and the attacker is modeled in hybrid Bayesian Stackelberg game setting. This framework combines the hierarchical decision-making of Stackelberg games with Bayesian probability to handle uncertainty about the attacker’s capabilities and objectives. The defender maintains a belief over attacker's profile parameters and anticipates their response when optimizing its own defense strategy. Let us derive their action spaces and utility functions.
% \begin{figure}[t]%
%     \centering
% \subfloat[\centering]{\includegraphics[width=0.49\linewidth,trim=1.2 0 0 1,clip]{Figures/fig1.png}}%
% \subfloat[\centering ]{\includegraphics[width=0.49\linewidth,trim=0 0 0 1,clip]{Figures/fig2.png}}%
%     \caption{Posterior mass over the attacker’s type parameter $\beta({\theta_{\text{type}}})$, of the defender's belief model in the case of (a) FDI and (b) MitM attacks on IEEE 13-Bus system.}%
% \label{fig:belief_convergence_types}
% \end{figure}%
\subsection{Attacker's action space}\label{attack_action_space}
We consider an intelligent attacker capable of launching three types of cyberattacks: FDI attacks, DoS, and MitM attacks. Each attacker is represented by a unified vector, hereafter referred as profile parameters, $\bm{\theta} = ( a, s, k, t_{0}, {\theta}_{\text{type}})$, where ${\theta}_{\text{type}} \in \{1, 2, 3\}$ indicates the type of attack (e.g., 1 for FDI and 2 for MitM). The parameter $a \in [0, 1]$ denotes the level of aggression, reflecting how strongly the attacker attempts to attack the system.
The parameter $s \in [0, 1]$ encodes stealthiness, representing how carefully the attacker hides its actions to avoid detection.
The parameter $k$ and $t_{0}$ controls the slope and activation time of the attack signal. 
The attacker can inject actions specific to FDI or a MitM attack as follows.
 \subsubsection{False Data Injection Attack} In the FDI attack, the attacker inserts false data into voltage observations $V_{d}$ of DER $d$ for some time $[t_0:t_{T_{A}}]$ \cite{deng2018false,bi2022impact}. For the $d-$th DER in the network, the replacement of actual voltage data $V_d(t)$ with $V^d_A(t)$.
     \begin{equation}\label{attack:fdi}      
     V_{\text{obs}}^d(t) = V_{\text{true}}^d(t) + V_A^d(t).
    \end{equation}  
   % (Action space is not the correct term here. Actions space you define for the attacker as a whole) action space is denoted by $\bm{a}_{A}^{d}=[\Delta \mathbf{V}_{A}^{d}, \mathbf{T}_{A}^{d}]$, subsequently, the attacker's action space at an instance $t$ for a DER $d$ can be defined as $a^{d}_{A}(t)= [\Delta V^{d}_{A}(t), t_{A}^{d}]$.
   To inject the false data, the action space of the attacker in the case of FDI attack is defined as:  
   \begin{equation}
       a^{d}_{A}(t)= [ V_{A}^{d}(t), t_{A}^{d}].
   \end{equation} 
 \subsubsection{Man-in-the-Middle Attack} In the case of a MitM, the attacker can modify the power output setpoints of the $d-$th DER: $(P^d(t), Q^d(t))$ \cite{raja2024mitigation,xia2012secure}
     \begin{equation}\label{attack:mitm}
     \small
     P_{\text{set}}^d(t) = P_{\text{legit}}^d(t) + P_A^d(t), \quad  Q_{\text{set}}^d(t) = Q_{\text{legit}}^d(t) + Q_A^d(t).      \end{equation}  
To manipulate their active and reactive power injections, the action space of the attacker includes:  
\begin{equation}
\small
     a^{d}_{A}(t)= [P_{A}^{d}(t), Q_{A}^{d}(t), t_{A}^{d}]. 
\end{equation}  

\subsubsection{Denial of service attack}
% 2.5
{
In a DoS attack, the adversary does not alter the physical values, but instead disrupts the availability of measurements by intermittently blocking or delaying their transmission \cite{huseinovic2020survey,wang2017strategic}. As a result, the defender does not receive fresh observations and instead operates on stale information.
\begin{equation}\label{attack:dos} 
V_{\text{obs}}^{d}(t)=\bigl(1 - \psi_{\mathrm{DoS}}(t)\bigr)V_{\text{true}}^{d}(t)
+  \psi_{\mathrm{DoS}}(t) V_{\text{obs}}^{d}(t-1),
\end{equation}
where $\psi_{\mathrm{DoS}}(t) \in \{0,1\}$ is a binary availability indicator.}

We fold the two behavioral parameters into one scalar envelope depending on aggression (a) and stealth (s): $\alpha=a(1-s)$. 
The attack signals can be designed as 
\begin{equation}\label{attackmap}
 \small
\mathbf{a}_A^d(\bm{\theta},t)=
\begin{aligned}
&[\,0,\;0,\;\alpha V_{\max}^d \psi_V(t)\,]^{\!\top} 
\text{ if } \theta_{\mathrm{type}}=\mathrm{FDI},\\[0.4em]
&[\,\alpha P_{\max}^d \psi_P(t),\;\alpha Q_{\max}^d \psi_Q(t),\;0\,]^{\!\top} 
\text{ if } \theta_{\mathrm{type}}=\mathrm{MitM}, ,\\[0.4em]
&[\,0,\;0,\; \alpha\psi_{\mathrm{DoS}}(t)\, V_{\mathrm{obs}}(t-1)\,]^{\!\top}, \quad \text{if } \theta_{\mathrm{type}}=\mathrm{DoS}.
\end{aligned}   
\end{equation}
The time-shaping functions \( \psi_{V}(t),\,\psi_{P}(t),\,\psi_{Q}(t) \in [-1, 1] \) allow modeling of bursts, ramps, or slowly varying biases. A typical choice is
    $\psi(t) = \tanh\bigl(k(t - t_0)\bigr)$,
which creates a smooth step. 
\begin{remark}
We acknowledge that real-world attacks can be more complex than the parametric profiles evaluated here, including replay attacks, coordinated multi-DER manipulation, and adaptive attackers attempting to suppress defensive signatures, as well as communication imperfections (delay, loss, jitter). Our simulations focus on stylized models of setpoint (measurement) tampering FDI (MitM) and availability disruption (DoS) for controlled, interpretable stress-testing; these attack surfaces align with established DER (VPP) cyber threat models \cite{wlazlo2021mitm}. Importantly, our detector follows an \emph{active watermarking} principle, injecting a small secret excitation and verifying its footprint, which is known to improve detectability of stealthy integrity and replay attacks \cite{yaghooti2021physicalwatermarking}. Coordinated multi-DER attacks are likewise recognized in active distribution networks and can be incorporated by extending the attacker action model (see Sec.~\ref{multider}) \cite{liu2024cfdi}. 
\end{remark}
\subsection{Defender's action space}
The defender (VPP) uses perturbation-based attack detection explained in Section \ref{sec3}, originally proposed in \cite{JPN21} to detect compromised DERs. 
The defender needs to decide on the magnitude of the real and reactive power perturbations $(|\Delta P^d(t)|, |\Delta Q^d(t)|)$ for each actor DERs $d \in \mathcal{D}_{\text{act}}$ and length of the perturbation sequence $T^{d}$ as well as vigilance factor $\rho$. 
The defender’s action space is defined as
\begin{equation}
    \mathbf{a}_D^d(t) = \left[\,\lvert\Delta P_{D}^{d}(t)\rvert,\; \lvert\Delta Q_{D}^{d}(t)\rvert,\; t_{D}^{d},\; \rho^{d}(t)\,\right],
\end{equation}
where $\lvert\Delta P_{D}^{d}\rvert$ and $\lvert\Delta Q_{D}^{d}\rvert$ represent the magnitudes of pseudo-random perturbations added to the actual power setpoints, $t_{D}^{d}$ denotes the perturbation timestamp, and $\rho^{d}(t)\in[0,1]$ is the vigilance factor controlling the defender’s sensitivity to potential attacks.
 
Let us derive the utility functions of the attacker ($U_{A}$) and the defender ($U_{D}$) by modeling their interactions in a Stackelberg Game framework with the leader being the defender and the attacker being the follower. 
\subsection{Utility Functions}
\subsubsection{Attacker's Reward}
The attacker’s goal is to maximize the impact or disruption $I(\cdot)$ of the attack on the network while minimizing the intrinsic probability of detection $p_{\textrm{det}}(\cdot)$ and the cost of attack $\mathcal{C}(\cdot)$. Without loss of generality, we suppress the DER index $d$ in the utility functions below.
\begin{multline}
   U_A(\mathbf{a}_{A}(\bm{\theta}, t))= \\
 -\hat{p}_{\mathrm{det}}(\mathbf{a}_{A}( \bm{\theta}, t)) + I(\mathbf{a}_{A}(\bm{\theta},t)) - k \mathcal{C}(\mathbf{a}_{A}(\bm{\theta}, t)).
\end{multline}
where $\hat{p}_{\mathrm{det}}(\mathbf{a}_{A}) = f_{\psi}(\mathbf{a}_{A})$, with 
$f_{\psi}(\mathbf{a}_{A}) \mapsto [0,1]$, denotes the detection probability, 
and the impact of an attack is modeled as
\begin{equation}
\small
    I(\mathbf{a}_{A}) = c_{1} f_{1}( V_{A}) + c_{2} f_{2}( P_{A}) + c_{3} f_{3}(Q_{A}).
\end{equation}
In FDI attacks, the impact is relatively limited since the VPP can directly observe anomalies in voltage measurements. 
In contrast, man-in-the-middle (MitM) attacks alter the DER active and reactive power set points, producing broader disruptions across both operations and market commitments. 
Specifically, in an FDI attack, false data can mislead the VPP into taking incorrect control actions, whereas in a MitM attack, malicious changes in $P_{\mathrm{set}}$ and $Q_{\mathrm{set}}$ can result in lost revenue and penalties for failing to meet contractual market obligations. 
To quantify these effects, we learn the impact functions $f_{1}(\Delta V_{A})$ for FDI and $f_{2}(\Delta P_{A})$, $f_{3}(\Delta Q_{A})$ for MitM, weighted by predefined constants $c_{1}, c_{2}, c_{3}$. 
The constants $c_{1}, c_{2}, c_{3}$ are empirically chosen to reflect the relative operational importance of
voltage, active-power, and reactive-power deviations, respectively. We adopt the ordering
$c_{3} > c_{2} > c_{1}$ to encode their system-level impact: \emph{reactive-power} disturbances ($\Delta Q$) most strongly affect voltage stability and can propagate across weakly coupled networks, risking wide-area voltage collapse; \emph{active-power}
deviations ($\Delta P$) directly perturb energy balance and frequency control and incur dispatch/market penalties, producing both operational and economic consequences; by contrast, \emph{voltage} magnitude deviations ($\Delta V$) are typically local effects of underlying $P/Q$ mismatches, are readily observable, and are often correctable by local regulation, thus receiving the smallest weight. 

% We assume the attacker attempts to learn the voltage deviation $\delta V_{A}$ induced by their own injected actions $\{ \Delta P_{A}, \Delta Q_{A}\}$ whenever executing a MitM attack, which attacker knows in the case of FDI attack. This estimate is obtained through a learned function $f_{\phi}(\Delta P_{A}, \Delta Q_{A}, \bm{\theta}_{\text{type}})$, where $\phi$ denotes the attacker’s internal model parameters.
% Given this, the attacker can also infer the defender’s contribution to the voltage deviation as:
% \[
% \delta \hat{V}_{\text{defender}} = \delta V_{o} - \delta \hat{V}_{A}
% \]
Finally, we define the cost incurred by the attacker as the squared norm of their actions  injections:
\begin{equation}
\mathcal{C}(\mathbf{a}_{A}) =  \mathbf{a}_{A}^\top \bm{W}(\bm{\theta}) \, \mathbf{a}_{A},
\end{equation}
where $\bm{W}(\bm{\theta})$  is a weighting matrix that assigns different cost sensitivities to each type of attack parameters $\bm{\theta}$.  
\subsubsection{VPP’s (Defender’s) Reward}
The VPP aims to design perturbation signals and their lengths maintaining a belief about attacker's profile $\bm{a}_A(\hat{\bm{\theta}})$ and detect DERs under attack effectively:
As a result, the defender's utility function in the game is defined as:
\begin{multline}\label{def_utility}
U_D\bigl(\bm a_D(t),\bm a_A({\hat{\bm{\theta}}},t)\bigr)
= \alpha\bigl(\rho(t)\bigr)\,p_{\textrm{det}}\bigl(\bm a_D,\bm{a}_A(\hat{\bm{\theta}},t)\bigr)\,
\\  R_{\mathrm{catch}}\bigl( \bm{a}_A(\hat{\bm{\theta}},t)\bigr) 
- \beta\bigl(\rho(t)\bigr)\,\mathcal C\bigl(\bm a_D(t)\bigr),
\end{multline}
where the first term determines the probability of detection, 
$p_{\mathrm{det}}\bigl(\bm{a}_{D}(t), \bm{a}_{A}(\hat{\bm{\theta}},t)\bigr) 
= P^{\text{attack}}_{\mathrm{det}}\bigl(\bm{a}_{D}(t)\bigr) 
f_{\varphi}\bigl(\bm{a}_{A}(\hat{\bm{\theta}},t)\bigr),$ 
weighted by the vigilance factor $\rho(t)$, 
with weighting functions defined as 
$\alpha(\rho(t)) = 1 + k_{1}\rho(t)$ and $\beta(\rho(t)) = 1 - k_{2}\rho(t).$ 
We model the defender’s intrinsic detector 
$f_{\varphi}\bigl(\bm{a}_{A}(\hat{\bm{\theta}},t)\bigr)$ as a neural network,
\begin{equation}
  \label{eq:detector-model}
  f_{\varphi}: \mathbb{R}^3 \longrightarrow [0,1].
\end{equation}
Here, $P^{\text{attack}}_{\mathrm{det}}$ denotes the probability of detection obtained from the detector module, 
and $f_{\varphi}$ explicitly quantifies the likelihood of detection based on the attacker’s input signals. The $f_{\varphi}$ makes the probability of attack type detection learnable. The network parameters $\varphi$ can be updated online as the attacker behavior changes. Thus the detection term $f_{\varphi}$ co-learns alongside statistical detector's $P_{\text{det}}$.
The second term, 
$R_{\text{catch}}\bigl(\bm{a}_{A}(\hat{\bm{\theta}},t)\bigr)
= c_{1} + c_{2}\exp{\bigl(I(\bm{a}_{A}(\hat{\bm{\theta}},t))\bigr)},$
represents the reward for detection and encourages the defender to prioritize identifying the most harmful attacks. The perturbations $\Delta P_{D}$ and $\Delta Q_{D}$ essentially correspond to small deviations in active and reactive power injections from the VPP’s actor DERs. In reality, these cannot fluctuate arbitrarily since large or frequent perturbations could violate power-quality limits or stress inverter.   
The third term quantifies the cost of the perturbation as 
$\mathcal{C}(\bm{a}_{D}, t) = \lVert \Delta P_{D}(t), \Delta Q_{D}(t) \rVert^{2}_{2}$, 
scaled by the vigilance factor $\rho(t)$. The quadratic cost penalizes such excessive control effort, keeping actions physically feasible.

\section{Bayesian Stackelberg Game Formulation}\label{sec5} 
% 2.4
{
A classical Stackelberg game assumes complete information, where the leader knows the follower’s utility and strategy space, yielding a deterministic leader–follower equilibrium. In contrast, a Bayesian Stackelberg game models uncertainty in the follower’s type through a belief distribution over attacker profiles, which is well suited to cyber–physical security settings where the defender does not know the attacker’s capabilities or intent a priori. By updating this belief using detection outcomes, the defender adapts its strategy over time, enabling robust and anticipatory defense under attacker uncertainty.}

In the Bayesian Stackelberg framework, the defender commits to an attack detection strategy before an attacker. However, the defender faces uncertainty about the attacker's profile $\bm{\theta} \in \Theta$, which is captured by a belief distribution $\beta(\bm{\theta})$.   

\subsection{Framework}
The interaction between the defender (VPP) and the attacker is modeled as a Stackelberg game where 
the {defender}, as the leader, chooses strategy $\bm{a}_{D}$, i.e. magnitude of the perturbation signal for each DER and length of the perturbation sequence to maximize its utility $U_{D}(\cdot)$. The defender does not know the true profile of the attacker $\bm{\theta}= [a, s , k, t_{0},{\theta}_{\mathrm{type}}]$. Instead, they maintain a belief distribution $\hat{\beta}(\bm{\widehat{\theta}}).$ 

Based on the prior belief, the VPP implements detection strategy:
    \begin{equation}\label{eq:defender_stackelberg}
    \bm a_D^*
=
\arg\max_{\bm a_D}
\sum_{\widehat{\bm\theta}}
\widehat{\beta}(\widehat{\bm\theta})
U_D\!\left(
\bm a_D,
\widehat{\bm a}_A^*(\widehat{\bm\theta})
\right).
    \end{equation}
The {attacker}, as the follower, observes the outcome of the defender's committed strategy $\bm{a}^{*}_{D}$, then chooses its best response, based on its type ${\theta}_{\text{type}}$, $\bm{a}^{*}_{A}$ to maximize its own utility:
    \begin{equation}\label{eq:attacker_best_response}
    \bm{a}^{*}_{A} = \arg\max_{\bm{a}_{A}} \; U_A(\bm{a}_{A}(\bm{\theta}), \bm{\widehat{\theta}}).
    \end{equation}
\noindent\textbf{Note.} We denote the true attacker profile by $\boldsymbol{\theta}$. The defender’s estimate of this profile is written as $\widehat{\boldsymbol{\theta}}$, obtained via the Gaussian process belief‐update procedure described in the next section. {The attack parameters are used only in the game-theoretic and belief-update layers to represent the uncertainty defender maintains regarding possible attacker behavior. These parameters are not used by the detector to determine whether the $V_P$ signal is absent.}

\subsection{Existence of a Bayesian Stackelberg Equilibrium}

The optimization in \eqref{eq:defender_stackelberg}--%
\eqref{eq:attacker_best_response} assumes that, for every
episode and defender belief, the attacker admits a best response
and the defender can select a belief-weighted optimal commitment.
Before integrating this game-theoretic solution with the RL
policy, we establish that the proposed Bayesian Stackelberg game
is well posed. In particular, the physical limits on DER
perturbations and attack magnitudes induce bounded action spaces,
while the proposed detection, impact, and quadratic-cost terms
yield continuous utilities. Under these conditions, an
equilibrium-informed defender action exists at every episode.

{
\begin{theorem}\label{equillibrium}
For episode \(\tau\), let
\[
\bm{a}_{D}\in \mathcal{A}_D\subset\mathbb{R}^{m_D},
\qquad
\bm{\theta}\in\Theta\subset\mathbb{R}^{m_\theta}
\]
be nonempty compact sets, and let
\(\beta^\tau\in\mathcal{P}(\Theta)\) be the defender's belief over attacker
profiles. For each \(\bm{\theta}\in\Theta\), let
\(\mathcal{A}_A(\theta)\subset\mathbb{R}^{m_A}\) be a nonempty compact attacker
action set. Assume that the correspondence
\(\bm{\theta}\mapsto \mathcal{A}_A(\bm{\theta})\) is continuous and compact-valued, and that
\[
U_D,U_A:\mathcal{A}_D\times \mathcal{A}_A(\bm{\theta})\times\Theta\rightarrow\mathbb{R}
\]
are bounded and continuous in their arguments on the feasible action-profile set.
Then there exists a Bayesian Stackelberg equilibrium
\[
\bigl(\bm{a}_D^{*,\tau},\bm{a}_A^{*,\tau}(\bm{\theta})\bigr),
\]
where
\[
\begin{aligned}
\bm{a}_A^{*,\tau}(\theta)
&\in BR_\theta(\bm{a}_D^{*,\tau}),\\
BR_\theta(\bm{a}_D)
&=
\arg\max_{\bm{a}_A\in\mathcal{A}_A(\bm{\theta})}
U_A(\bm{a}_D,\bm{a}_A,\theta).
\end{aligned}
\]
and
\[
\bm{a}_D^{*,\tau}
\in
\arg\max_{\bm{a}_D\in\mathcal{A}_D}
\int_{\Theta}
U_D\!\left(\bm{a}_D,\bm{a}_A^{*,\tau}(\bm{\theta}),\bm{\theta}\right)
\,d\beta^\tau(\bm{\theta}).
\]
\end{theorem}
}

\begin{proof}
{
Fix an episode \(\tau\) and a defender belief
\(\beta^\tau\in\mathcal{P}(\Theta)\). In the proposed Stackelberg game, the VPP
defender is the leader and first commits to a perturbation strategy
\(\bm{a}_D\in\mathcal{A}_D\). After observing this commitment, an attacker with profile
\(\bm{\theta}\in\Theta\) acts as the follower and selects a best response
\[
BR_\theta(\bm{a}_D)
=
\arg\max_{\bm{a}_A\in\mathcal{A}_A(\bm{\theta})}
U_A(\bm{a}_D,\bm{a}_A,\bm{\theta}).
\]
Since \(\mathcal{A}_A(\bm{\theta})\) is nonempty and compact and \(U_A\) is continuous
over the follower's feasible action set, the Weierstrass theorem guarantees that
\(BR_\theta(\bm{a}_D)\) is nonempty for every \((\bm{a}_D,\theta)\)
\cite{sundaram1996first}. Moreover, by Berge's maximum theorem, the
best-response correspondence has nonempty compact values and is upper
hemicontinuous under the stated compactness and continuity assumptions
\cite{berge1963topological}.}
{
We adopt the standard strong Stackelberg convention: when several attacker actions yield the same follower utility, the selected best response is the one most favorable to the leader. This tie-breaking rule is used only to resolve indifference among the follower's best responses; it does not alter the attacker's objective or assume cooperative behavior. Hence, for each \(\bm{\theta}\), define the defender's induced value as
\begin{equation*}
V_D(\bm{a}_D,\bm{\theta})
=
\max_{\bm{a}_A\in BR_{\bm{\theta}}(\bm{a}_D)}
U_D(\bm{a}_D,\bm{a}_A,\bm{\theta}).
\end{equation*}
Because $BR_\theta(\bm{a}_D)$ is compact-valued and \(U_D\) is continuous, another application of the Weierstrass theorem ensures that this maximum is attained. Berge's maximum theorem further implies that \(V_D(a_D,\theta)\) is upper semicontinuous in \(\bm{a}_D\) \cite{berge1963topological}.}

{ Now define the defender's belief-weighted objective
\[
J_D^\tau(\bm{a}_D)
=
\int_{\Theta} V_D(\bm{a}_D,\theta)\,d\beta^\tau(\bm{\theta}).
\]
Since \(U_D\) is bounded and \(V_D\) is upper semicontinuous in the leader action,
\(J_D^\tau(\bm{a}_D)\) is well-defined and upper semicontinuous on the compact action
space \(\mathcal{A}_D\). Therefore, by the Weierstrass theorem, there exists at
least one maximizer
\[
\bm{a}_D^{*,\tau}
\in
\arg\max_{\bm{a}_D\in\mathcal{A}_D} J_D^\tau(\bm{a}_D).
\]
For each attacker profile \(\theta\), choose a follower response
\[
\bm{a}_A^{*,\tau}(\bm{\theta})
\in
\arg\max_{\bm{a}_A\in BR_\theta(\bm{a}_D^{*,\tau})}
U_D(\bm{a}_D^{*,\tau},a_A,\theta).
\]
By construction,
\(\bm{a}_A^{*,\tau}(\theta)\in BR_\theta(\bm{a}_D^{*,\tau})\), and
\(\bm{a}_D^{*,\tau}\) maximizes the defender's expected utility under the induced
strong Stackelberg follower response. Hence,
\[
\bigl(\bm{a}_D^{*,\tau},\bm{a}_A^{*,\tau}(\bm{\theta})\bigr)
\]
constitutes a Bayesian Stackelberg equilibrium at episode \(\tau\).
}
\end{proof}

\begin{remark}
The conditions of Theorem~\ref{equillibrium} guarantee the existence, but not the uniqueness, of the Bayesian Stackelberg equilibrium. The strong Stackelberg convention resolves multiplicity in the follower's best response by selecting,
among utility-maximizing attacker actions, the response most favorable to the defender. However, the resulting belief-weighted defender objective \(J_D^\tau(a_D)\) may still admit multiple maximizing actions. In this case, the GT module selects one admissible equilibrium maximizer returned by the numerical solver, and the selected action is supplied to the RL policy as equilibrium-informed guidance. The RL agent then adapts this guidance using the current system observations and its reward objective.
\end{remark}

% where $U_D(\bm{a}_{D})= \sum_{t=1}^{T-1}U_D(\bm{a}_{D},t)$ and $U_A(\bm{a}_{A},\bm{a}_{D})= \sum_{t=1}^{T-1} U_A(\bm{a}_{A},\bm{a}_{D},t)$

\subsection{Defender's Belief $\beta(\bm{\theta})$}
{
% \subsubsection{Attacker Type Vector Representation}
Each attacker is represented by the profile parameters $\bm{\theta} = [ a, s,  k, t_{0}, {\theta}_{\text{type}}]$ defined in the Section~\ref{attack_action_space}. 
The defender maintains a belief distribution $\beta(\bm{\theta})$ over the attacker's profile parameters, which is updated over time based on observed system behavior and attack patterns.}
% \begin{figure*}[!b]
% \vspace{-0.6cm}
% \noindent\makebox[\linewidth]{\rule{\textwidth}{0.4pt}}
% \vspace{-0.6cm}
% \begin{multline}\label{eq20}{34}
%   U_D\!\bigl(\bm a_D(t),\bm a_A(\widehat{\bm\theta},t)\bigr)
% =\alpha(\rho_t)\,P_{\mathrm{det}}\!\bigl(\bm a_D(t)\bigr)\,
% \underbrace{f_{\varphi}\!\bigl(\bm a_A(\widehat{\bm\theta},t)\bigr)\,
% (\bar c_{1}+\bar c_{2}\,e^{I(y)}\!\bigl(\bm a_A(\widehat{\bm\theta},t)\bigr))}_{=:~g(\bm a_A(\widehat{\bm\theta},t))}
% -\beta(\rho_t)\,\mathcal C\!\bigl(\bm a_D(t)\bigr), 
% \end{multline}
% \vspace{-0.2cm}
% with $R_{\mathrm{catch}}(y)=\bar c_{1}+\bar c_{2}\,e^{I(y)}$ and 
% $I(y)=c_1 f_1(\Delta V_A)+c_2 f_2(\Delta P_A)+c_3 f_3(\Delta Q_A)$.
% \begin{equation}\label{eq21}{35}
%    \big|U_D(\bm a_D,y)-U_D(\bm a_D,y')\big|
% \le \alpha(\rho_t)\,P_{\mathrm{det}}(\bm a_D)\,(R_{\max}L_f+L_R)\,L_A\,\|\widehat{\bm\theta}-\bm{\mu}_{\tau}\|_2
% \le \alpha_{\max}P_{\max}(R_{\max}L_f+L_R)L_A\,\|\widehat{\bm\theta}-\bm{\mu}_{\tau}\|_2. 
% \end{equation}
% \end{figure*}
This formulation embeds all attacker types into a common  latent space. Regardless of the underlying mechanism of each attack type, their observable characteristics are modeled consistently along the behavioral dimensions $\bm{\theta}_{\text{behave}}= [a, s,  k, t_{0}]$ of aggressiveness and stealth.

\subsubsection{Modeling the Belief with a Gaussian Process}
The defender’s belief over attacker behavior is modeled as a Gaussian process (GP) defined over the unified input space $\bm{\theta} \in \mathbb{R}^{5}$. This belief function is modelled as:
\begin{equation}\label{gp}
  f(\bm{\theta}) \sim \mathcal{GP}(m(\bm{\theta}), k(\bm{\theta}, \bm{\theta}')).  
\end{equation}
The function $f(\bm{\theta})$ outputs a belief score, which represents the likelihood that an attacker with profile $\bm{\theta}$ is responsible for the observed behavior. The mean function $m(\bm{\theta})$ reflects prior knowledge, often initialized as zero when no prior bias exists. The kernel function $k(\bm{\theta},\bm{\theta}')$ captures similarity between attacker profiles and governs the smoothness of the belief surface in the attacker space.

This modeling approach enables the defender to generalize across attacker types and behaviors, while systematically allowing belief updates through the GP posterior.
\subsubsection{Choosing the Kernel Function}
To define the belief function as a GP over the unified attacker space $\bm{\theta} = ( a, s, k, t_0, {\theta}_{\text{type}})$, we need to construct a composite kernel that captures both categorical type similarity in $\bm{\theta}_{\text{type}}$ and behavioral proximity in ${a, s}$. The kernel is expressed as:
\begin{equation}
k(\bm{\theta}, \bm{\theta}') = k_{\text{type}}({\theta}_{\text{type}}, {\theta}'_{\text{type}}) \cdot k_{\text{behave}}(\bm{\theta}_{\text{behave}}, \bm{\theta}_{\text{behave}}').
\end{equation}
The function $k_{\text{type}}$ handles the similarity between attacker types. It may be defined as an indicator function (for e.g. 1 if the types are identical, $\rho\in [0,1]$ otherwise), or as a smooth kernel (e.g., exponential decay) if some similarity is assumed between types.

The behavioral kernel $k_{\text{behave}}$ measures similarity in the continuous $\bm{\theta}_{\text{behave}}$ space. A common choice is the squared exponential (RBF) kernel:
\begin{equation}\label{eq16}
\small
k_{\mathrm{behave}}\bigl(\bm{\theta}_{\mathrm{behave}},\,\bm{\theta}_{\mathrm{behave}}'\bigr)
= \exp\!\Bigl(-\tfrac{1}{2\ell^2}\,
\|\bm{\theta}_{\mathrm{behave}} - \bm{\theta}_{\mathrm{behave}}'\|^2\Bigr),
\end{equation}
This structure enables the GP to learn belief functions that reflect smooth variations in attacker behavior, while distinguishing between different attack types.

% --- Place in preamble if not already present ---
% \usepackage{amsthm,amsmath,amssymb}
% \newtheorem{lemma}{Lemma}
{
\begin{lemma}\label{lem:mixed_kernel_psd_rho}
Let $\bm{\theta}=(\bm{x},c)$ where $\bm{x}\in\mathbb{R}^d$ and
$c\in\mathcal C$ is categorical. Consider the kernel
\begin{equation}
k(\bm{\theta},\bm{\theta}')
= k_{\text{type}}(c,c')\,k_{\text{behave}}(\bm{x},\bm{x}'),
\end{equation}
where
\[
k_{\text{type}}(c,c')=
\begin{cases}
1,& c=c'\\
\rho,& c\neq c'
\end{cases},
\quad 0\le\rho\le1,
\]
\[
k_{\text{behave}}(\bm{x},\bm{x}')
=\exp\!\left(-\frac{\|\bm{x}-\bm{x}'\|^2}{2\ell^2}\right).
\]
Then $k$ is positive semidefinite on $\mathbb{R}^d\times\mathcal C$ and
therefore defines a valid Gaussian process covariance function.
\end{lemma}}
% \vspace{-0.8cm}
{
\begin{proof}
Rewrite
\[
k_{\text{type}}(c,c')=\rho+(1-\rho)\delta(c,c'),
\]
where $\delta(c,c')=\mathbb{I}[c=c']$. Both the constant kernel and the
delta kernel are PSD, hence $k_{\text{type}}$ is PSD for $\rho\in[0,1]$.
The RBF kernel $k_{\text{behave}}$ is PSD on $\mathbb{R}^d$.
The composite Gram matrix satisfies
$\mathbf K=\mathbf K_{\text{type}}\circ\mathbf K_{\text{behave}}$.
By the Schur product theorem, $\mathbf K\succeq0$, hence $k$ is PSD.
\end{proof}}

\subsubsection{Updating $\beta(\bm{\theta})$ based on observations}
As the system interacts with the environment during an episode $\tau$, the defender observes the outcomes associated with specific configurations of attacker behavior. Each interaction is represented as $\bm{\theta}^{\tau} = (a^{\tau}, s^{\tau}, k^{\tau}, t_{0}^{\tau},{\theta}_{\text{type}}^{\tau})$, along with detector outputs $\bm{o}^{\tau}=[{C_{\mathrm{det}}^{\text{attack}}}^{\tau}, {P_{\mathrm{det}}^{\text{attack}}}^{\tau}, {\theta}_{\mathrm{det}}^{\tau} ]$, that is, attack detection confidence, probability of detection of attack, and detector label for the type of attack, respectively, forming a data set: $\mathcal{D} = \{(\bm{\theta}_i, \bm{o}_i)\}_{i=1}^{T}$.

The Gaussian Process (GP) prior defined in \eqref{gp} is then conditioned on the  dataset $\mathcal{D}$ to yield a posterior belief function:
\begin{equation}
 \beta(\bm{\theta}) \mid \mathcal{D} \sim \mathcal{N}(\mu_*(\bm{\theta}), \sigma_*^2(\bm{\theta})),   
\end{equation}
where $\mu_*(\bm{\theta})$ represents the updated belief score, a smooth estimate of how likely an attacker with profile $\bm{\theta}$ is to explain the observed outcomes. The term $\sigma_*^2(\bm{\theta})$ captures the epistemic uncertainty in this belief, indicating the confidence the model has about each region of the attacker space.

{\color{black}
The posterior samples are summarized by their mean \(\bm\mu_{\tau}\) and
covariance \(\bm\Sigma_{\tau}\). These moments define the Gaussian belief
propagated to the next episode, which is the
\(\mathrm{MomentMatch}(\cdot)\) operation in Algorithm~\ref{alg:gt-rl}. The
hidden continuous parameters \([a,s,k,t_0]\) are retained for attack generation
and offline belief-tracking diagnostics; the defender-side update uses the
detector outputs \(\bm o^\tau\).
}

\section{Reinforcement Learning Model Formulation}\label{sec6}  
The RL agent finds optimal perturbation magnitudes by maximizing the reward function $r(\bm{a}_{RL})$. Consequently, the action space of the RL agent is defined as $\bm{a}^{d}_{RL}(t)=[\Delta P_{RL}^{d}(t),\Delta Q_{RL}^{d}(t), t_{RL}^{d}]$. 

{The state space for the RL agent is just the set of instantaneous voltage measurements at actor nodes, their variance, and a risk factor ($\rho$) : $\bm{z}(t) = \left[V_o(t),\sigma^2(t), \rho(t)\right]$.}

The RL agent interacts with the current environment and decides upon the perturbation magnitude to inject $\bm{a}_{RL}(t)= [\Delta P_{RL}, \Delta Q_{RL}, t_{RL}]$, following a pre-defined one-time pad as the base sequence to scale for detection.
% \subsection{Reinforcement Learning agent's reward function} 
We define the RL agent's reward function as:
\begin{align}\label{eq:rl_reward}
  \notag & r(\bm{z}(t), \bm{a}_{\text{RL}}(t)) = \mathcal{A}(\bm{a}_{RL},t) -  \mathcal{C}(\bm{a}_{RL},t) - \notag\\  &   \lambda_D \left( \left\| \bm{a}_{\text{RL}}(t)
  - \bm{a}_{D}^{*^{\text{perturb}}}(t) \right\|_2^2 \right) 
  - \lambda_\rho  \rho(t)\left\| \bm{a}_{\text{RL}}(t) \right\|_2^2, 
\end{align}
where $\mathcal{A}(\bm{a}_{\text{RL}},t)$ denotes the attack‐detection term defined as 
$\mathcal{A}(\bm{a}_{\text{RL}},t) = \bigl(p(\mathcal{H}_1)\bigr) \cdot 
\begin{cases}
1, & \text{if } (|\mathcal{T}_t| > \gamma)\,|\,\bm{x}\neq0,\\
0, & \text{otherwise},
\end{cases}$
with $\mathcal{T}_t$ representing the detection threshold and $\gamma(P_{\mathrm{FA}})$ the false‐alarm threshold derived from the probability of false alarm $P_{\mathrm{FA}}$.  
This term quantifies whether the detector’s confidence in identifying the injected perturbation sequence exceeds the false‐alarm threshold.  

The cost of the perturbation selected by the RL agent is given by 
$\mathcal{C}(\bm{a}_{\text{RL}},t) = k_i \lVert \Delta V_o(t) \rVert_2^2 + k_c \lVert \Delta P_{\text{RL}}(t), \Delta Q_{\text{RL}}(t) \rVert_2^2$, 
which measures the total system‐voltage deviation in response to the agent’s power injection.  
Here, $\bm{a}_{D}^{*}$ represents the optimal perturbation strategy derived from the game‐theoretic model, defined as $\bm{a}_{D}^{*} = [\bm{a}_{D}^{*^{\text{perturb}}}, \rho(t)]$, and $\lambda_D$ is the game‐theoretic alignment weight.  
Finally, the term $\lambda_\rho\, \rho(t)\lVert \bm{a}_{\text{RL}}(t) \rVert_2^2$ penalizes the RL agent for selecting DERs that have a high likelihood of being under attack.
% \subsection{Policy Update}
For policy updates, we implement proximal policy optimization (PPO) \cite{schulman2017proximal} and twin delayed deep deterministic policy gradient (TD3) \cite{fujimoto2018addressing}.

%3.3
{Perturbation magnitudes are explicitly penalized in the GT module through a risk-weighted quadratic cost $\mathcal{C}$ in \eqref{def_utility}, which discourages physically minimum power possible actions. Moreover, all perturbations are implicitly bounded by inverter capability limits on active and reactive power injections $ P_{\min}^d \le |\Delta P_{RL}^d(t)| \le P_{\max}^d,\;
Q_{\min}^d \le |\Delta Q_{RL}^d(t)| \le Q_{\max}^d$. All modern inverters allow external programs to read inverter ratings, parameters, and measurements via communication protocols such as SunSpec Modbus and DNP3 at up to 1s time resolution. These protocols also support modifying the inverter's active and reactive power setpoints, making these perturbations practical to implement. The RL agent's actions 
$\bm{a}_{\text{RL}}$ is further regularized to align with the GT perturbations $\bm{a}_{D}^{*^{\text{perturb}}}$, ensuring that the learned policy injects only small, low-power setpoint adjustments consistent with practical grid operation.}
\begin{algorithm}[t]
\caption{Bayesian Stackelberg Game + RL Defense}\label{alg:gt-rl}
\begin{algorithmic}[1]
\State Initialize belief $(\boldsymbol{\mu},\boldsymbol{\Sigma})$, episode $\tau\gets0$
\For{$\tau=1:T_{\max}$}

\State Obtain detector output:
$\hat{\mathbf{o}}^{\tau-1}\leftarrow
\begin{cases}
\mathrm{Detector}(\mathbf{s}^{\tau-1},\bm a_{\mathrm{RL}}^{\tau-1}) & \text{if available}\\
f_\varphi(\bm a_A^{\tau-1}) & \text{otherwise}
\end{cases}$

\State GP likelihood:
$p_{\mathrm{GP}}(\hat{\mathbf{o}}^{\tau-1}\mid\boldsymbol{\theta})
=\mathcal N(\boldsymbol{\mu}_{\mathrm{GP}},\boldsymbol{\Sigma}_{\mathrm{GP}})$

\State Sample attacker profile:
$\boldsymbol{\theta}^{\tau}\sim
\mathrm{HMC}\big(\log[\mathcal N(\boldsymbol{\theta};\boldsymbol{\mu},\boldsymbol{\Sigma})
\,p_{\mathrm{GP}}]\big)$

\State Defender Stackelberg action:
$\bm a_D^{\tau}=\arg\max_{\bm a_D}\beta(\boldsymbol{\theta})
U_D(\bm a_D,\bm a_A^*(\boldsymbol{\theta}^{\tau}))$

\State RL action:
$\bm a_{\mathrm{RL}}^{\tau}\sim\pi_\phi(\bm z^{\tau-1},\bm a_D^{\tau})$

\State Attacker best response:
$\bm a_A^{\tau}=\arg\max_{\bm a_A}U_A(\bm a_A(\boldsymbol{\theta}),\widehat{\boldsymbol{\theta}})$

\State Environment transition:
$(\mathbf{s}^{\tau},r^{\tau-1})\leftarrow
\mathrm{Env}(\mathbf{s}^{\tau-1},\bm a_{\mathrm{RL}}^{\tau},\bm a_A^{\tau})$

\State Update detector nets:
$\psi\leftarrow\mathrm{TrainDetector}(\psi),\quad
\varphi\leftarrow\mathrm{TrainDetector}(\varphi)$

\State Belief update:
$(\boldsymbol{\mu},\boldsymbol{\Sigma})\leftarrow
\mathrm{MomentMatch}(p(\boldsymbol{\theta}\mid\hat{\mathbf{o}}^{1:\tau}))$

\State RL policy update:
$\phi\leftarrow\mathrm{RLUpdate}(\phi)$

\EndFor
\end{algorithmic}
\end{algorithm}
\begin{figure}[t]
    \centering
    \includegraphics[width=1\linewidth]{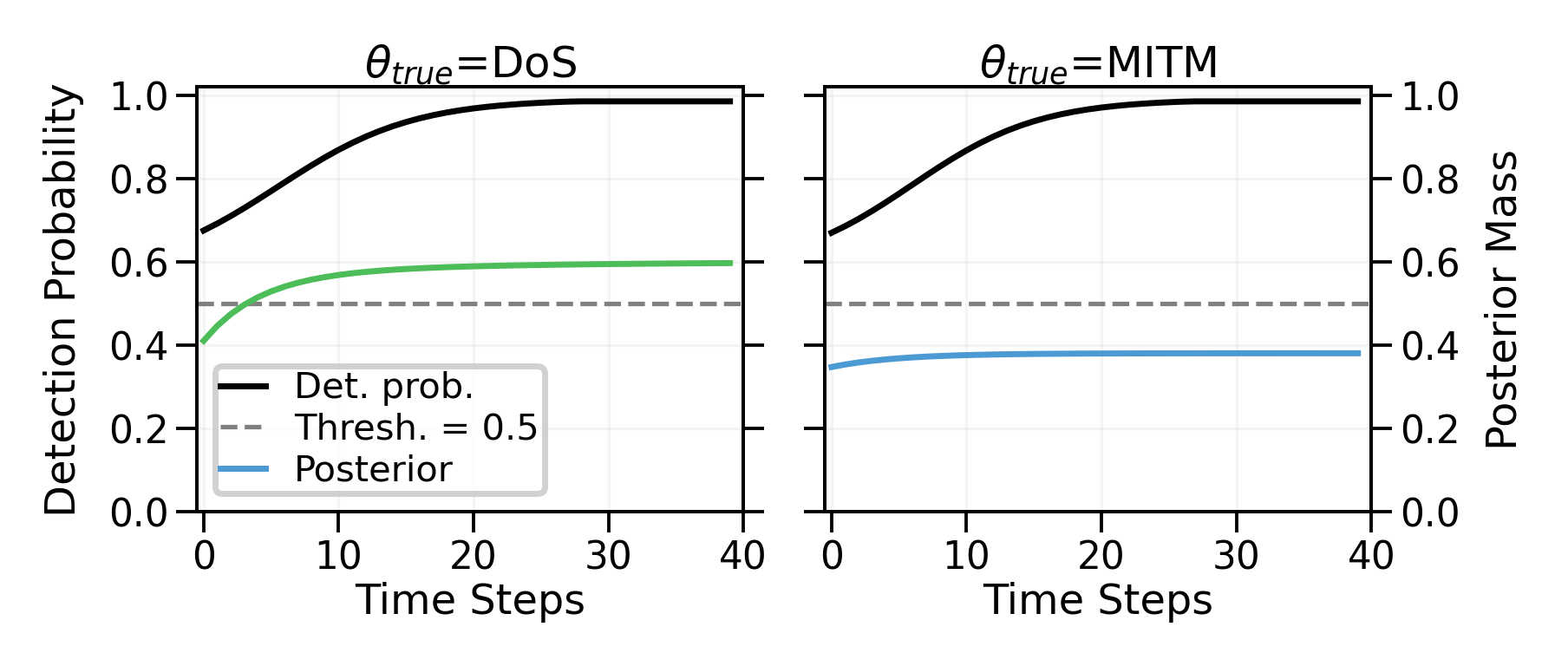}
    \caption{Defender’s detection probability on the IEEE 13-Bus system.}\label{det_mitm}
\end{figure}

{\color{black}
\textcolor{black} {
The Bayesian Stackelberg formulation in Section~V establishes that, for the
current defender belief, an equilibrium-informed perturbation action exists.
The RL formulation above then specifies how this game-theoretic action enters
the realized defense through the shaped reward in~\eqref{eq:rl_reward}.
It remains to characterize how uncertainty in the defender's belief and the
RL realization of the Stackelberg-guided action jointly affect the resulting
defender utility. The following theorem therefore connects the
moment-matched belief \((\bm{\mu}_{\tau},\bm{\Sigma}_{\tau})\), the
corresponding Stackelberg perturbation, and the perturbation ultimately
realized by the RL policy. It provides an episode-level bound on the resulting
utility gap in terms of two quantities already present in the framework:
attacker-profile uncertainty represented by the posterior covariance and
deviation of the RL action from the Stackelberg-guided perturbation controlled
by the alignment weight \(\lambda_D\).
}

\begin{theorem}\label{thm:regret_param_error}
Let \(\beta^\tau(\bm\theta)=\mathcal N(\bm\mu_\tau,\bm\Sigma_\tau)\) be the
defender belief after moment matching. Let
\(\bm a_D^{*^{\mathrm{perturb}}}(\bm\mu_\tau,t)\) be the perturbation component
of the belief-guided Stackelberg action, and let \(\bm a_{RL}^\tau(t)\) be the
bounded perturbation action that maximizes the existing shaped reward
\[
\mathcal A(\bm a,t)-\mathcal C(\bm a,t)
-\lambda_D\|\bm a-\bm a_D^{*^{\mathrm{perturb}}}(\bm\mu_\tau,t)\|_2^2
-\lambda_\rho\rho(t)\|\bm a\|_2^2.
\]
Define the one-step utility gap for \(\bm\theta\sim\beta^\tau\) as
\begin{equation}
\begin{split}
r_t^\tau
={}&U_D\!\left(\bm a_D^*(\bm\theta,t),
                 \bm a_A^*(\bm\theta,t)\right)\\
&-U_D\!\left(\bm a_{RL}^\tau(t),
                 \bm a_A^*(\bm\mu_\tau,t)\right).
\end{split}
\end{equation}
Assume that the induced Stackelberg utility is \(L_\theta\)-Lipschitz in the
attacker profile, and that the defender utility and the part of the shaped
reward excluding the GT-alignment term are \(L\)-Lipschitz in the bounded
perturbation action. Then
\begin{equation}\label{eq37}
\|\bm a_{RL}^\tau(t)
-\bm a_D^{*^{\mathrm{perturb}}}(\bm\mu_\tau,t)\|_2
\le \frac{L}{\lambda_D},
\end{equation}
and
\begin{equation}\label{eq38}
\mathbb E_{\beta^\tau}[|r_t^\tau|]
\le
L_\theta\sqrt{\operatorname{tr}(\bm\Sigma_\tau)}
+\frac{L^2}{\lambda_D}.
\end{equation}
\end{theorem}

\begin{proof}
Because \(\bm a_{RL}^\tau(t)\) maximizes the shaped reward, its reward is no
smaller than the reward obtained by the feasible comparison action
\(\bm a_D^{*^{\mathrm{perturb}}}(\bm\mu_\tau,t)\). The GT-alignment penalty is
zero for the comparison action. Therefore,
\[
\lambda_D
\|\bm a_{RL}^\tau(t)-
\bm a_D^{*^{\mathrm{perturb}}}(\bm\mu_\tau,t)\|_2^2
\le
L\|\bm a_{RL}^\tau(t)-
\bm a_D^{*^{\mathrm{perturb}}}(\bm\mu_\tau,t)\|_2,
\]
which gives \eqref{eq37}. Next, add and subtract the defender utility obtained from the
belief-guided Stackelberg action. The attacker-profile contribution is bounded
by \(L_\theta\|\bm\theta-\bm\mu_\tau\|_2\), using the attacker map in
\eqref{attackmap} and the induced-utility Lipschitz assumption. The action
contribution is bounded by
\(L\|\bm a_{RL}^\tau(t)-
\bm a_D^{*^{\mathrm{perturb}}}(\bm\mu_\tau,t)\|_2\). Moment matching gives
\[
\mathbb E_{\beta^\tau}
\|\bm\theta-\bm\mu_\tau\|_2
\le
\sqrt{\mathbb E_{\beta^\tau}
\|\bm\theta-\bm\mu_\tau\|_2^2}
=
\sqrt{\operatorname{tr}(\bm\Sigma_\tau)}.
\]
Combining this relation with \eqref{eq37} yields \eqref{eq38}.
\end{proof}

}

\textcolor{black} {
Theorem~\ref{thm:regret_param_error} formalizes the closed-loop coupling between belief inference,
game-theoretic guidance, and RL action selection. The first term in the bound
quantifies the contribution of uncertainty in the inferred attacker profile,
while the second quantifies the deviation between the realized RL perturbation
and the Stackelberg-guided action. Hence, improved concentration of the
defender belief and stronger GT--RL alignment tighten the corresponding terms
of the bound. The result is an episode-level characterization of the existing
GT--RL mechanism and does not assert global convergence of the nonlinear RL
training dynamics.
}

\section{Game-RL Hybrid Architecture}\label{sec7}  
{
The proposed architecture couples Bayesian Stackelberg reasoning with RL agent through both the policy input and the learning objective. At each episode, the game-theoretic module uses the current belief over attacker profiles to compute an equilibrium-informed perturbation action. This action is then provided to the RL policy as a strategic guidance signal and is also incorporated into the reward through an alignment term. Consequently, the RL agent does learns to adapt the Stackelberg-guided action using real-time voltage observations, detector feedback, and evolving operating conditions. The sequential Bayesian Stackelberg formulation provides a tractable abstraction of the leader--follower interaction over the episode interval \([t_\tau,T]\), while the RL policy translates this strategic guidance into an adaptive perturbation action for the physical DER environment.}

\subsubsection{Defender’s Commitment (Phase 1) $[t_\tau:t_{D}]$}
% \item Let $1:T$ denote the duration of the episode for the RL agent estimates the perturbations based on GT's suggestions.  
At the start of the episode, the defender maintains a belief $\beta^\tau(\bm{\theta})$ over attacker's profile. Let us lay a GP prior $\beta^\tau(\bm{\theta}) \sim GP(m(\bm{\theta}), k(\bm{\theta},\bm{\theta}'))).$
Based on this belief and the current system state $s_0$ (e.g., voltage levels, DER status), a Stackelberg Game theoretic model proposes a candidate action:
\begin{align}
     \notag {\bm{a}}_D^{\tau} & = \arg\max_{\bm{a}_{D}} \sum_{\bm{\theta}} \beta^{\tau}(\bm{\theta}) \cdot \sum_{d \in \mathcal{D}} \bigg( \sum_{t=t_{\tau}}^{t_{D}} U_D(\bm{a}_{D}, \hat{\bm{a}}_A^*(\bm{\theta},t)) \\
      & + \sum_{t=t_{A}}^{T} \mathbb{E}(\hat{U}_{D}(\bm{a}_{D}, \hat{a}_A^*(\bm{\theta},t)) \bigg)
\end{align}
However, the final perturbation strategy $\bm{a}_{RL}^\tau$ is selected by the defender's RL policy, which adapts over time and chooses how much to lean on the GT suggestion. Formally, 
\begin{equation}
 \bm{a}_{\mathrm{RL}}^{\tau} \sim \pi_{\phi}(\bm{a}|\bm{z}, {\bm{a}}_D^{\tau}).
\end{equation}
The actions $\bm{a}_{RL}^{\tau}$ for phase 2 are applied to the network. \\
{\textbf{Note.} The prescribed \(P_{FA}\) in detector controls the false-alarm probability of the
\(V_P\)-signal hypothesis test. An attack false positive occurs when the expected \(V_P\) signal is present under benign operation but is declared absent or
inconsistent. The belief-update and game-theoretic modules reduce such errors indirectly by guiding the RL agent toward more informative perturbation magnitudes, thereby improving the distinguishability of the
\(V_P\) signal from voltage fluctuations and reducing signal misses.}

\subsubsection{Attacker’s Response (Phase 2) $[t_{A}:T]$}
During the phase 2, the perturbations signals from the defender as well as the attack's signals are present. 
{The attacker, whose true type $\bm{\theta}_{\text{true}} \sim \mathcal{N}(\bm{\mu}_{\text{true}}, \mathbf{\Sigma}_{true})$ is unknown to the defender}, {observes whether the voltage packets it injected have been modified (in the case of FDI) or the active/reactive power setpoints it manipulated have been overwritten by the controller (in the case of MitM). This allows the attacker to infer the presence of defensive monitoring and adapt its action accordingly.} Given these observations $\tilde{\bm{o}}$ and its internal model, the attacker selects its optimal strategy:
\begin{figure}[t]
    \centering
    \includegraphics[width=1.02\linewidth]{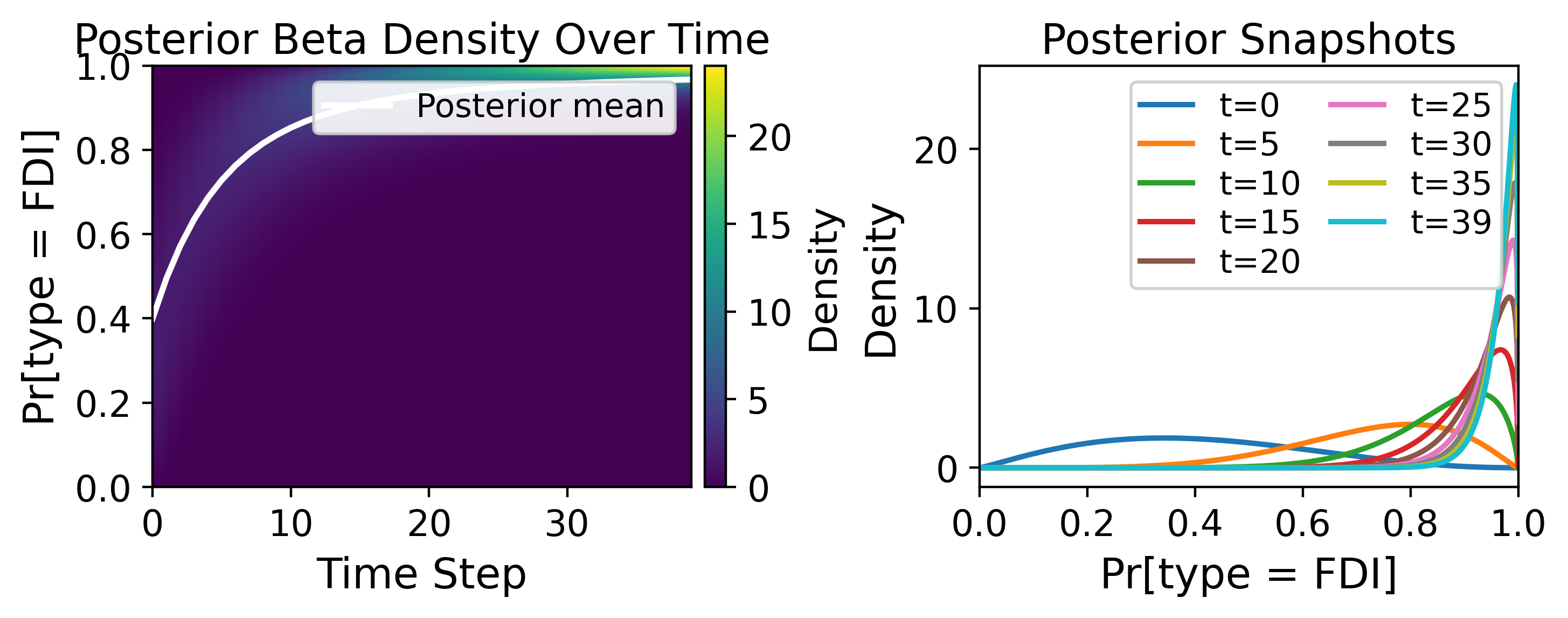}
    \caption{Defender’s evolving belief distribution on the IEEE 13-Bus system.
The dashed lines show the $\pm 2\sigma_{*}({\theta}_{\mathrm{det}})$ region around the mean of posterior belief distribution, $\mu_{*}({\theta}_{\mathrm{det}})$.}\label{fig:perturbation_magnitude_change}
\end{figure}
\begin{equation}
    \mathbf{a}^{*}_A = \arg\max_{\mathbf{a}_{A}} \; \sum_{t_{A}}^{T} U_A(\mathbf{a}_{A}(\bm{\theta}_{\mathrm{true}},t), \tilde{o}^{(\tau-1)}),
    \end{equation}
    where $\tilde{o}^{(\tau-1)}$ are the observations detection labels obtained from the defender's detector module from the prior episode $(\tau-1)$.
\subsubsection{Detection, belief, and policy update}
 Following the attacker’s response, the system state evolves. The defender continuously monitors the system during phase 2 via its detector, which outputs a measurement vector: 
\begin{equation}
    \mathbf{o}^{\tau} = [C_{\mathrm{det}},\, P_{\mathrm{det}},\, {\theta}_{\mathrm{det}}] \in \mathbb{R}^{3d},
\end{equation}
forming a dataset: $\mathcal{D} = \{\bm{\theta}^{\tau}, \bm{o}^{\tau}\}$. 
\begin{figure}
    \centering
\includegraphics[width=1.02\linewidth]{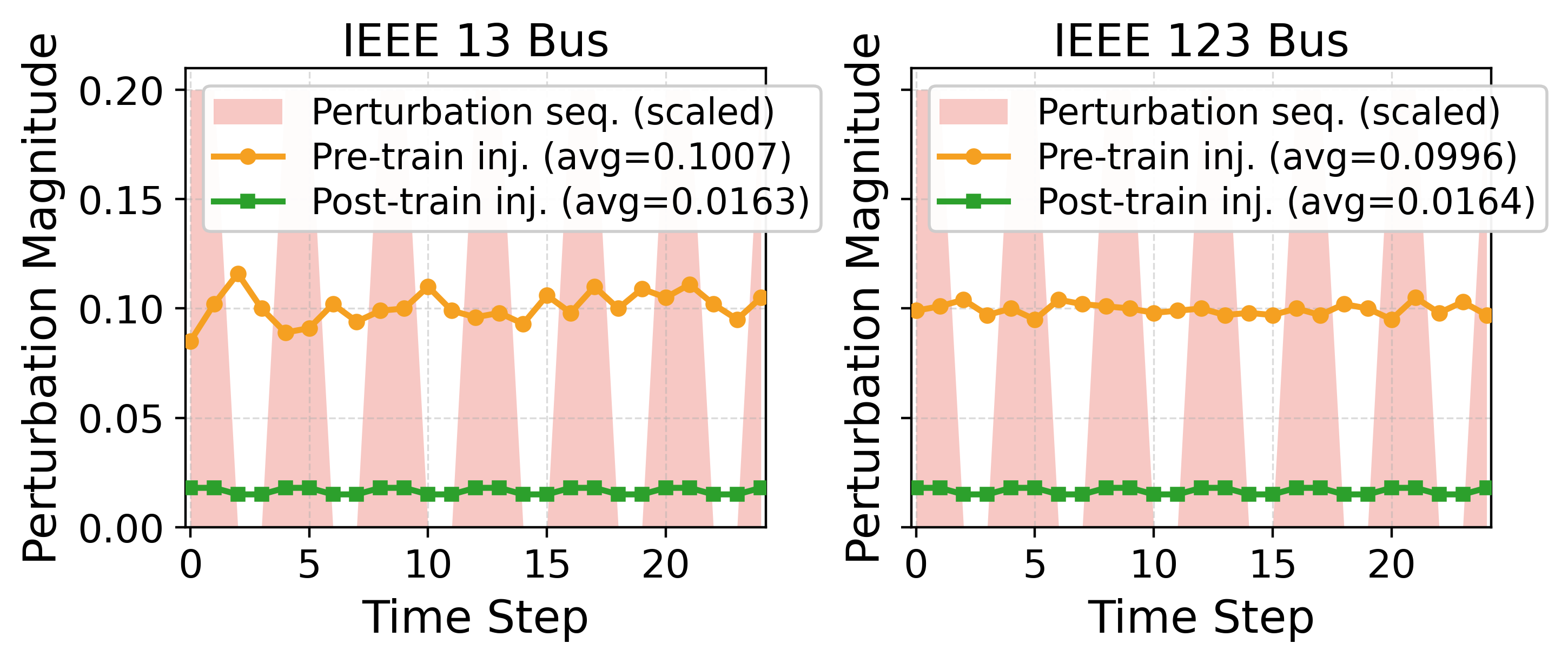}
    \caption{Convergence of reinforcement learning agent behavior across IEEE 13-Bus and 123-Bus systems. Learned perturbation magnitude decreases by an order of magnitude with no impact to detection performance.}
    \label{fig:rl_convergence}
\end{figure}
In this vector, $C_{\mathrm{det}}$ represents the detection confidence, $P_{\mathrm{det}}$ denotes the base detection probability, and ${\theta}_{\mathrm{det}}$ contains attack labels across $d$ dimensions. Using Bayes’ rule to incorporate the new evidence, the defender updates its posterior belief over the potential attack types (assuming $m(\bm{\theta})$ to be zero).

The attacker's profile for the next episode $\bm{\theta}^{\tau+1}$ in the defender's GT can be estimated using maximum a posteriori as follows: 
\begin{equation}
    \bm{\theta}_{\text{MAP}} = \arg\max_{\bm{\theta}} \, p\left(\hat{\beta}^{\tau+1}(\bm{\theta}) \mid \mathcal{D}, l\right),
\end{equation}
where $l$ is the parameter of RBF kernel function in \eqref{eq16}.
The proposed GT-RL method is illustrated in Fig. \ref{depict} and summarized in algorithm \ref{alg:gt-rl}.

{
{\subsection{Potential Extension to Coordinated Multi-DER Attacks}}\label{multider}
A fully general coordinated adversary compromising \(N\) DERs can be modeled by a vector-valued action
\(\bm{a}_A(t)=[a_{A,1}(t),\ldots,a_{A,N}(t)]^\top\), where \(a_{A,i}(t)\) denotes the manipulation applied to DER \(i\) (e.g., measurement tampering for FDI, command/setpoint overwrite for MiTM, or packet dropping for DoS). In its unconstrained form, the attacker best response becomes
\[
\bm{a}_A^{*}(\tilde{\bm{a}}_D^\tau)
=\arg\max_{\bm{a}_A}\sum_{t=t_A}^{T}U_A\!\big(\tilde{\bm{a}}_D^\tau,\bm{a}_A(t),\bm{\theta}_{\text{true}},t\big)\]
such that $\|\bm{a}_A(t)\|\le A_{\max}$,
which is high-dimensional and may be intractable for online equilibrium computation. To retain realism while preserving tractability, coordinated multi-DER attacks can be parameterized in a structured low-dimensional form. One practical representation is a \emph{budgeted subset} model: the attacker selects a subset \(S\subset\{1,\ldots,N\}\) with \(|S|\le B\) (coordination budget) and applies a shared parametric profile \(g(t;a,k,t_0)\), i.e.,
\(a_{A,i}(t)=\mathbb{I}[i\in S]\cdot \alpha_i\, g(t;a,k,t_0)\),
where \(\alpha_i\) are optional per-DER scaling factors; this naturally reuses our parameter \(s\) as a sparsity/budget proxy (e.g., \(B=\lceil sN\rceil\)). An even more compact formulation uses a \emph{low-rank/mode} model,
\[
\bm{a}_A(t)=\sum_{r=1}^{R} w_r\,\bm{u}_r\, g_r(t;\,a_r,k_r,t_{0,r}),
\qquad R\ll N,
\]
where \(\{\bm{u}_r\}\) are fixed coordination patterns (e.g., DER clusters or learned modes) and \(\{w_r\}\) are mixing weights. Under either parameterization, the GT-RL defender formulation extends directly: the attacker profile \(\bm{\theta}\) is augmented to include coordination variables (e.g., \(B\) and a cluster/mode index \(r\), or weights \(\{w_r\}\)), the Stackelberg layer computes a candidate perturbation \(\tilde{\bm{a}}_D^\tau\) under the current belief \(\beta^\tau(\bm{\theta})\), and the RL policy outputs the realized perturbation action \(\bm{a}_{RL}^\tau\sim\pi_\phi(\bm{a}\mid \bm{z},\tilde{\bm{a}}_D^\tau)\). Importantly, the detection mechanism remains footprint-based: it tests consistency of the defender's secret perturbation signature across observer voltages; coordinated attacks are captured by the expanded attacker action model, while belief updates incorporate detector outputs to infer the coordination budget/mode and adapt perturbations accordingly without requiring explicit topology knowledge. 

The coordinated multi-DER formulation is not
implemented or evaluated in the present study; it is included only to outline
a possible extension of the attacker action model.}

\begin{figure*}[t]
    \centering
    \includegraphics[width=\textwidth]
    {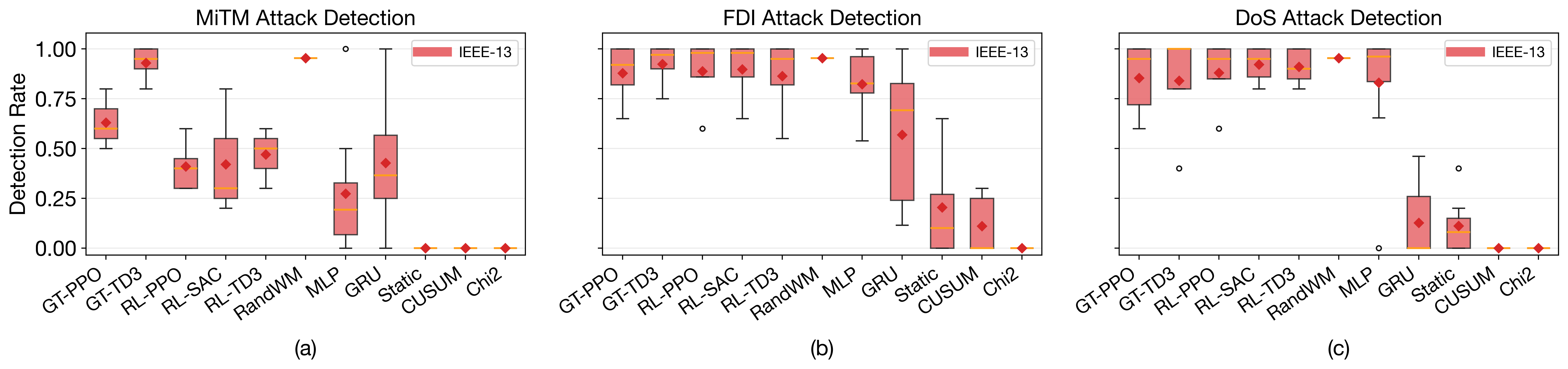}
    \caption{IEEE-13 single-attack baseline comparisons for MiTM, FDI, and DoS attacks. The proposed GT-RL methods are compared with RL-only, fixed-watermarking, machine-learning-based, and residual/statistical detection baselines.}
    \label{fig:ieee13_single_attack_baselines}
\end{figure*}
\begin{figure*}[t]
    \centering
    \includegraphics[width=\textwidth]
    {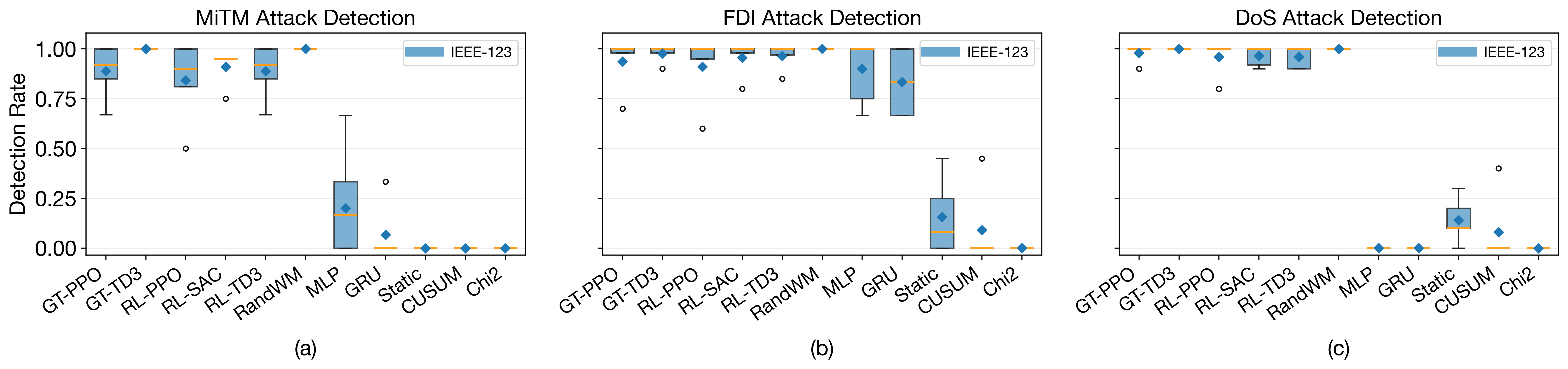}
    \caption{Ablation study contrasting attack detection rates with and without game-theoretic parameter tuning of the RL agent, alongside conventional statistical approaches: (a) MiTM; (b) FDI; (c) DoS attack detection for IEEE 123-Bus system.}
    \label{ablation_123}
\end{figure*}

\section{Simulation and Results}\label{sec8} 

To validate the proposed methodology, we examine whether the attacker’s uncertainty model, formulated via the Stackelberg game within the VPP defense mechanism, enhances the defender’s ability to detect the true attack type. We consider three attack scenarios: MitM, FDI, and DoS. The attacker’s profile parameters $\bm{\theta}$ are mapped to attack signals for the attack scenarios using the design equations in \eqref{attackmap}.
{The IEEE 13-bus and IEEE 123-bus test systems were augmented with inverter-interfaced DERs placed at selected load buses, following standard practices in distribution system studies. Specifically, seven DERs were randomly deployed in the IEEE 13-bus system and sixty-two DERs in the IEEE 123-bus system. Each DER is capable of modulating active and reactive power injections within predefined limits.}

{DERs are coordinated through a centralized VPP abstraction, which aggregates individual DER responses and exposes a unified control interface to the defender. The proposed GT-RL agent operates at this VPP level, determining perturbation magnitudes that are then distributed across actor DERs.
Furthermore, we designed the true attack signals by introducing biases into the underlying distributions of attacker's profile parameters, $\bm{\theta}_{\text{true}}$ allowing us to stress-test the resilience of the models under adversarial conditions.}

% For DoS case, the attacker can manipulate time index or availability without injecting any physical values on the entire action space of the attacker $a^{d}_{A}(t)= [P_{A}^{d}(t), Q_{A}^{d}(t), V_{A}^{d}(t)] \; ;
% \forall d \in \mathcal{D}_{act} $ . 
% 2.5
% { Therefore,
% in practice, we model DoS as a non-selective attack signals that affect all measurement and control channels, degrading availability and timing rather than injecting structured false values. In our implementation, this effect is captured through unstructured, low-energy perturbations that manifest as reduced detection confidence and higher entropy in the inferred attack signature.}
% 2.6
% #2.1

{The estimation of coefficients $\bm{x}=[x_{p}, x_{q}]$ for the detector is performed using \eqref{wls} using weighted least squares (WLS). Please note that their estimation is conducted in a data-driven manner only using the observed perturbation response data, voltage observations $\Delta \bm{V}_{o}$ and perturbations $\Delta \bm{P}, \Delta \bm{Q}$, which does not require the network topology.} {This WLS estimation assumes that, over the short perturbation interval, the voltage response to small setpoint perturbations is approximately linear and locally time-invariant, which is a standard approximation in distribution systems for small deviations around an operating point. Consequently, the perturbation magnitudes are intentionally kept small to remain in the linear regime, and the WLS estimation is performed using normal data to avoid bias from corrupted measurements during attacks.}

First, we show that as the VPP defender updates its belief about the attacker’s profile parameter \(\bm{\theta}\), the detection probability \(P_D\) of the perturbation-based detector improves. The posterior mass of the belief distribution \(\beta(\bm{\theta})\) increases for the true attack class in the case of the FDI attack (Fig.~\ref{det_mitm}). In Fig.~\ref{fig:perturbation_magnitude_change}, we illustrate how the belief distribution over the attack type parameter $\theta_{\text{type}}$ evolves over time for the IEEE 13-bus system. As the belief is updated with new observations, the uncertainty gradually decreases, and the posterior distribution converges toward the true attack type, as shown in the time-snapshot plot.

{
The proposed GT-RL variants are evaluated using two policy-update rules: (i) {GT-PPO}, which integrates Bayesian Stackelberg guidance with PPO; (ii) {GT-TD3}, which combines the GT module with TD3. These methods are compared against the corresponding RL-only baselines: (iii) {RL-PPO}, (iv) {RL-SAC}, and (v) {RL-TD3}. Additionally, we include (vi) {RandWM}, a fixed random watermarking baseline that injects pseudo-random active- and reactive-power perturbations under the same perturbation budget and applies the same perturbation detector; (vii) an {MLP}, which learns window-level watermark-response consistency from aligned clean windows and synthetically corrupted windows; and (viii) a {GRU}, which uses the same training information while preserving the temporal structure of the perturbation-response sequence. The MLP and GRU baselines are trained without oracle access to MITM, FDI, or
DoS labels. Positive samples are aligned defender-perturbation and voltage-
response windows obtained under benign operation. Negative samples are
generated from benign data by masking, delaying, shuffling, or mismatching the
response window, thereby breaking the expected $V_P$-signal relationship. Both
models are trained using binary cross-entropy to estimate $V_P$-signal
consistency. During evaluation, a low consistency probability indicates a
missing or distorted $V_P$ signal. \textcolor{black} {
All compared methods use the same attack distributions, DER sets, detector
windows, and evaluation protocol. Adaptive perturbation methods are subject to
the same admissible active- and reactive-power range and the same hard
DER-capability limits; no method is permitted to exceed these bounds.
}} 

{
We organize the evaluation into two experimental classes: single-attack detection and mixed-attack classification.}

{
In the single-attack setting, the true attacker uses only one attack class throughout the evaluation experiment: FDI-only, MiTM-only, or DoS-only. Each attack class is evaluated separately on the IEEE-13 and IEEE-123 feeders. The results in Fig.~\ref{fig:ieee13_single_attack_baselines} and
Fig.~\ref{ablation_123} therefore report the type-specific detection rate, i.e., the fraction of evaluated attack intervals of the corresponding true attack type that are correctly detected. In this setting, the proposed GT-RL methods, especially the GT-TD3 variant, perform competitively with or better than the RL-only and classical baselines across the MiTM, FDI, and DoS cases.}

{
In the mixed-attack setting, each evaluation episode contains multiple fixed-duration attack intervals, and the active attack type can change across intervals.} 
{
We further evaluate three stress regimes. In the adaptive-attacker regime, the true distribution of the attacker profile parameters evolves across evaluation
episodes. The defender does not observe these parameters directly; instead, it updates its belief from detector feedback through the GP-based Bayesian belief
module. This tests whether the proposed GT-RL defense can track a nonstationary attacker and adjust its perturbation policy accordingly.}

{
In the topology-change regime, the RL policy, detector calibration, and sensitivity model are obtained on the base feeder topology, but evaluation is
performed on a reconfigured topology. For IEEE-13, the evaluation topology is formed by opening the native 671--692 connection and adding an alternate
632--692 tie. For IEEE-123, the evaluation topology is formed by opening Switch~5 and reconfiguring Switch~7 to connect buses 151 and 300. This changes the feeder connectivity and voltage-response sensitivities without changing
the attack-generation protocol. Finally, in the adaptive+topology regime, both sources of distribution shift are applied simultaneously: the attacker profile distribution evolves during evaluation, and the policy calibrated on the base topology is evaluated on the
reconfigured feeder. This setting jointly evaluates the method under cyber-side attacker nonstationarity and physical-side feeder reconfiguration.}

Tables~\ref{tab:mixed_macro_f1_ieee13} and~\ref{tab:mixed_macro_f1_ieee123}
summarize mixed-attack classification performance under adaptive-attacker, topology-change, and adaptive+topology stress regimes. The reported values are macro-F1 scores over the three attack labels, MiTM, FDI, and DoS, averaged over 10 evaluation runs. Thus, the metric evaluates whether the defender correctly classifies the active attack type. 
\textcolor{black} {
Across both feeders, the proposed GT-RL methods retain comparable
macro-F1 performance across these regimes, indicating no systematic
performance degradation under the evaluated attacker-distribution and feeder
configuration shifts. GT-TD3 achieves the highest mean macro-F1 among the
evaluated methods across all three conditions. }

The static-threshold, CUSUM, and \(\chi^2\) baselines are designed primarily
for residual-change detection. Their low macro-F1 values and detection rates indicate that residual magnitude alone does not reliably distinguish MITM, FDI, and DoS, particularly when the perturbation
regressors are corrupted, unavailable, or affected differently across attack
mechanisms. These baselines also use fixed residual thresholds and do not
incorporate attacker-profile beliefs or adaptive perturbation design.

{
Second, we evaluate whether the perturbation magnitudes obtained by the proposed GT-RL algorithm constitute effective decision variables for attack detection. Fig.~\ref{fig:convergence_parameters} illustrates the complete decision pathway for the IEEE-123 feeder. In the first panel, the defender's belief $\hat{\boldsymbol{\theta}}$ changes in response to the active MiTM, FDI, and DoS attacker profiles. Although the estimated aggression parameter is conservative relative to the true attacker mean, it tracks the changes in the active attacker profile and provides updated information to the Stackelberg optimization. }

{
As shown in the second panel, GT-TD3 selects consistently stronger perturbation magnitudes than RL-TD3. The game-theoretic layer uses the estimated attacker profile to influence the  magnitude and DER allocation of the active- and reactive-power perturbations, while the RL reward continues to penalize excessive control effort and voltage deviation. The resulting decision therefore represents an attack-aware compromise between perturbation visibility and its physical cost. The third panel shows that the GT-informed perturbations produce an average detector SNR gain of $0.14$~dB relative to RL-TD3. A higher SNR indicates that the perturbation-induced voltage response is more distinguishable from measurement noise and attack-induced distortion, thereby providing the detector with stronger evidence.
This improvement propagates to the final attack-type decision. As shown in the fourth panel, GT-TD3 achieves an average mixed-attack macro-F1 gain of $0.18$ over RL-TD3 across the evaluation episodes.}

% In early episodes, the defender's posterior over attacker parameters is diffuse; consequently, the RL policy injects larger perturbations to reliably elicit a measurable voltage ``footprint'' across a broad range of plausible attacker behaviors (i.e., a high exploration / high value-of-information regime). As episodes progress, the Stackelberg belief module assimilates detector outputs and the posterior concentrates toward the true attacker profile (as evidenced by decreasing divergence across episodes in Fig. \ref{fig:convergence_parameters}), reducing uncertainty and making perturbations more targeted and predictable. Under this sharper belief, the RL policy can exploit a consistent, type-specific footprint and therefore decrease perturbation amplitude while preserving (or improving) detectability. This trend is reflected in Fig. \ref{fig:rl_convergence}, where the average injection magnitude drops from \(\approx 0.10\) pre-training to \(\approx 0.016\) post-training across both IEEE-14 and IEEE-123.

\begin{figure}[t!]
    \centering
\includegraphics[width=0.99\linewidth]{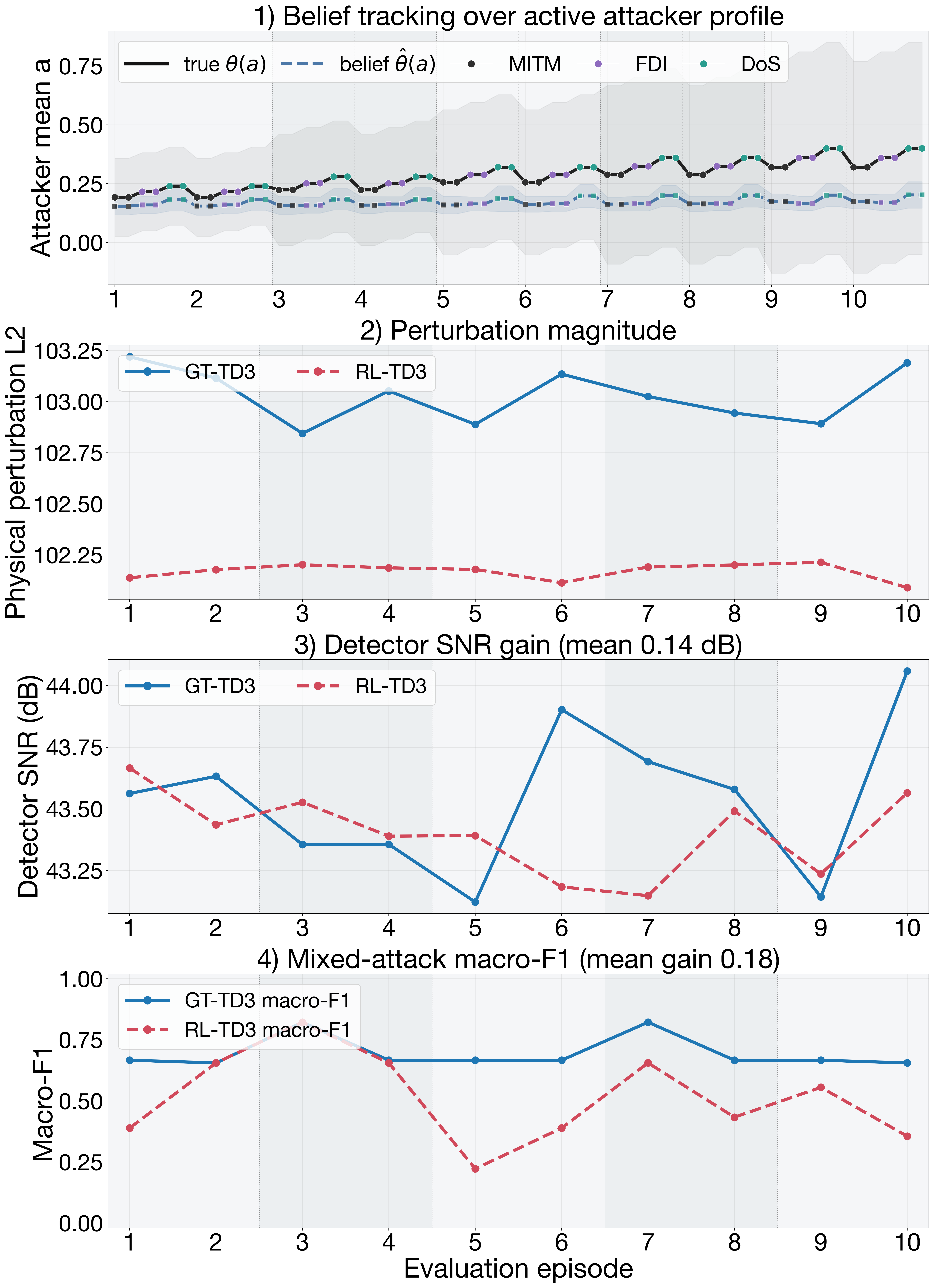}
    \caption{Impact of game-theoretic belief update on 1) attacker-profile parameter distribution on defender end; 2) the resulting perturbation magnitude obtained by RL agent; 3) detector SNR; and 4) mixed-attack macro-F1 across evaluation episodes for the IEEE-123 feeder under adaptive mixed attacks.}
    \label{fig:convergence_parameters}
\end{figure}  

\begin{table}[t]
\centering
\scriptsize
\setlength{\tabcolsep}{5pt}
\renewcommand{\arraystretch}{1.12}
\begin{tabular}{lccc}
\toprule
Method & Adaptive & Topology & Adaptive+Topology \\
\midrule
GT-PPO 
& 0.572 $\pm$ 0.203 
& 0.591 $\pm$ 0.138 
& 0.591 $\pm$ 0.138 \\

GT-TD3 
& \textbf{0.650 $\pm$ 0.136} 
& \textbf{0.699 $\pm$ 0.182} 
& \textbf{0.714 $\pm$ 0.192} \\

RL-PPO 
& 0.625 $\pm$ 0.169 
& 0.652 $\pm$ 0.149 
& 0.670 $\pm$ 0.134 \\

RL-TD3 
& 0.571 $\pm$ 0.158 
& 0.608 $\pm$ 0.136 
& 0.620 $\pm$ 0.154 \\

RL-SAC 
& 0.616 $\pm$ 0.132 
& 0.604 $\pm$ 0.165 
& 0.604 $\pm$ 0.165 \\

RandWM 
& 0.606 $\pm$ 0.081 
& 0.670 $\pm$ 0.112 
& 0.660 $\pm$ 0.110 \\

MLP & 0.507 $\pm$ 0.193 & 0.508 $\pm$ 0.193 & 0.481 $\pm$ 0.178 \\
GRU & 0.428 $\pm$ 0.121 & 0.506 $\pm$ 0.162 & 0.481 $\pm$ 0.145 \\

Static 
& 0.093 $\pm$ 0.151 
& 0.000 $\pm$ 0.000 
& 0.100 $\pm$ 0.161 \\

CUSUM 
& 0.168 $\pm$ 0.025 
& 0.149 $\pm$ 0.061 
& 0.161 $\pm$ 0.040 \\

$\chi^2$ 
& 0.093 $\pm$ 0.151 
& 0.000 $\pm$ 0.000 
& 0.100 $\pm$ 0.161 \\
\bottomrule
\end{tabular}
\caption{Mixed-attack macro-F1 for the IEEE-13 test system. Entries are mean $\pm$ standard deviation over $10$ evaluation runs. Bold marks the best method for each regime.}
\label{tab:mixed_macro_f1_ieee13}
\end{table}

\begin{table}[t]
\centering
\scriptsize
\setlength{\tabcolsep}{5pt}
\renewcommand{\arraystretch}{1.12}
\begin{tabular}{lccc}
\toprule
Method & Adaptive & Topology & Adaptive+Topology \\
\midrule
GT-PPO 
& 0.649 $\pm$ 0.236 
& 0.586 $\pm$ 0.182 
& 0.587 $\pm$ 0.192 \\

GT-TD3 
& \textbf{0.696 $\pm$ 0.067} 
& \textbf{0.717 $\pm$ 0.097} 
& \textbf{0.717 $\pm$ 0.097} \\

% GT-SAC 
% & 0.671 $\pm$ 0.178 
% & 0.648 $\pm$ 0.154 
% & 0.609 $\pm$ 0.173 \\

RL-PPO 
& 0.664 $\pm$ 0.230 
& 0.630 $\pm$ 0.241 
& 0.632 $\pm$ 0.255 \\

RL-TD3 
& 0.513 $\pm$ 0.184 
& 0.520 $\pm$ 0.170 
& 0.520 $\pm$ 0.170 \\

RL-SAC 
& 0.585 $\pm$ 0.191 
& 0.577 $\pm$ 0.203 
& 0.561 $\pm$ 0.178 \\

RandWM 
& 0.502 $\pm$ 0.241 
& 0.498 $\pm$ 0.232 
& 0.517 $\pm$ 0.203 \\

MLP & 0.386 $\pm$ 0.235 & 0.471 $\pm$ 0.238 & 0.389 $\pm$ 0.140 \\
GRU & 0.398 $\pm$ 0.252 & 0.327 $\pm$ 0.219 & 0.419 $\pm$ 0.157 \\

Static 
& 0.267 $\pm$ 0.141 
& 0.267 $\pm$ 0.141 
& 0.267 $\pm$ 0.141 \\

CUSUM 
& 0.163 $\pm$ 0.100 
& 0.199 $\pm$ 0.056 
& 0.163 $\pm$ 0.100 \\

$\chi^2$ 
& 0.267 $\pm$ 0.141 
& 0.267 $\pm$ 0.141 
& 0.267 $\pm$ 0.141 \\
\bottomrule
\end{tabular}
\caption{Mixed-attack macro-F1 for the IEEE-123 test system. Entries are mean $\pm$ standard deviation over $10$ evaluation runs. Bold marks the best method for each regime.}
\label{tab:mixed_macro_f1_ieee123}
\end{table}

% \begin{table*}[t]
% \centering
% \small
% \caption{Performance comparison between GT-RL methods, RL-only baselines, and classical statistical detectors across different cyber-physical attack types. Values are reported as mean (standard deviation).}
% \scalebox{0.85}{
% \begin{tabular}{lcccccccc}
% \toprule
% Attack Type & GT+RL (PPO) & GT+RL (TD3) & RL Only (PPO) & RL Only (SAC) & RL Only (TD3) & Static Threshold & CUSUM & Chi-Square \\
% \midrule
% MITM & 0.26 (0.03) & 0.24 (0.05) & 0.12 (0.01) & 0.15 (0.02) & 0.32 (0.06) & 0.00 (0.00) & 0.00 (0.00) & 0.00 (0.00) \\
% FDI  & 0.22 (0.03) & 0.23 (0.05) & 0.00 (0.00) & 0.00 (0.00) & 0.00 (0.00) & 0.00 (0.00) & 0.00 (0.00) & 0.00 (0.00) \\
% DoS  & 0.59 (0.01) & 0.53 (0.06) & 0.63 (0.04) & 0.56 (0.05) & 0.50 (0.04) & 0.01 (0.00) & 0.00 (0.00) & 0.00 (0.00) \\
% \bottomrule
% \end{tabular}
% }
% \label{topology_changes}
% \end{table*}

{
\paragraph*{Computational latency and real-time deployment}
The proposed GT-RL framework is structured to fit within real-time VPP operational cycles by decoupling \emph{offline training} from \emph{online execution}. In deployment, the VPP does not solve an RL optimization problem at each dispatch step; instead, it executes a single forward pass of the learned policy \(\pi_{\phi}(\cdot)\) to generate the perturbation action, followed by the perturbation detector evaluation and a lightweight belief update (GP posterior/MAP) over a low-dimensional attacker-parameter space. The game-theoretic (Stackelberg) guidance can be computed at a slower cadence than dispatch, for e.g., only when the belief shifts materially, and can be accelerated via warm-starting and a small candidate-action set. As a result, the online compute path (policy inference + detector + belief update) is compatible with typical VPP control loops, including 5-min dispatch intervals, while avoiding repeated topology-dependent optimization during operation.}
{
{\paragraph*{Limitations}
The approach may degrade when operating conditions change abruptly (e.g., large load/DER ramps) or when the feeder undergoes major reconfiguration that significantly alters the local sensitivities; in such cases, the learned coefficients $x_{p,o}$ and $x_{q,o}$ depend on the feeder configuration and local operating point and are therefore not assumed to remain permanently
fixed. Abrupt topology changes or large operating-point shifts may invalidate the current coefficients. In longer-term operation, or when persistent model mismatch is observed, the coefficients can be re-estimated from newly collected perturbation-response measurements without requiring explicit feeder-topology
information. The present analysis establishes equilibrium existence and bounded utility
degradation under belief misspecification, but does not establish global
convergence of the complete GT-RL learning dynamics with nonlinear function
approximation. Finally, since our formulation does not assume topology access, we do not claim invariance to arbitrarily large structural changes; rather, we demonstrate robustness to realistic topology variations and operating-point drift through the added topology-shift experiments and the use of short-horizon, data-driven coefficient updates.}}

{All experiments were conducted using Python and PyTorch, with reinforcement learning implemented via the PPO and TD3 algorithm in Stable-Baselines-3 using the \texttt{MlpPolicy} architecture. Training was initially performed with GPU acceleration on a workstation equipped with two NVIDIA RTX~4090 GPUs; however, due to the modest network and policy sizes, training could also be completed efficiently on CPU with negligible impact on runtime. Unless otherwise stated, the learning rate was set to $10^{-4}$ and identical hyperparameters were used across all test systems to ensure fair comparison.}

% We considered three types of attacks: FDI, MitM, and \textcolor{red}{denial of service (DoS)}. For each attack type, we generate a time-ramped attack sequence and track the defender's posterior mean for the action parameters over time. The main metric is the time-to-threshold, defined as the first time $t$ such that $\mu_{\text{true}}^{(t)} \geq 0.90$, where $\mu_{\text{true}}^{(t)}$ is the posterior mean for the true attack type $\theta_{true}$. 

% \section{Conclusion and Future work}
% In this paper, we propose a topology-agnostic, perturbation-based attack-detection method for VPP-operated DERs: a hybrid cybersecurity framework that combines Bayesian Stackelberg game theory with reinforcement learning (GT-RL) to enable adaptive defense under uncertainty. GT-RL achieves higher detection performance than advanced RL baselines. Using this framework, we show how VPPs can deploy a principled perturbation strategy to defend against a broad class of cyber-physical attacks, including FDI, MiTM, and DoS. As future work, we will extend the method to detect temporally correlated attack signals \cite{mcaleer2023team} and broaden coverage to additional threats such as replay attacks and topology-level FDI.

\section{Conclusion}
{
In this paper, we propose a topology-agnostic perturbationbased
attack-detection framework for VPP-operated DER
fleets that integrates Bayesian Stackelberg game reasoning
with model-free reinforcement learning (GT-RL) to enable
adaptive defense under attacker uncertainty. Across IEEE-
13 and IEEE-123 systems and multiple attack classes (FDI,
MiTM, DoS), GT-RL consistently outperforms standalone RL
baselines (PPO, TD3) and classical detectors while requiring
only local voltage observations available to the VPP. The
game-theoretic belief module progressively concentrates the
posterior over attacker parameters, providing strategic guidance
that improves detection and reduces the perturbation
effort over time.

The approach relies on a locally linearized response model
and sufficient voltage observability, and performance may
degrade under abrupt operating-point shifts, missing data,
or feeder reconfiguration. These limitations motivate periodic
recalibration of response coefficients and more robust
handling of nonstationarity. Future work includes explicit
evaluation under coordinated multi-DER and adaptive
signature-cancellation attackers, which are supported by the
structured/low-rank attacker action extensions described in
Sec.~\ref{multider}. In addition, the future work will also extend the
framework to detect temporally correlated \cite{liang2024game} and stealthier
attack patterns (e.g., replay and topology-level FDI), incorporate
richer observation models and change-point detection for
evolving feeder conditions, and develop stronger theoretical
guarantees on belief convergence and detection under bounded
model mismatch.

}
\bibliography{references}

\begin{IEEEbiography}[{\includegraphics[width=1in,height=1.25in,clip,keepaspectratio]{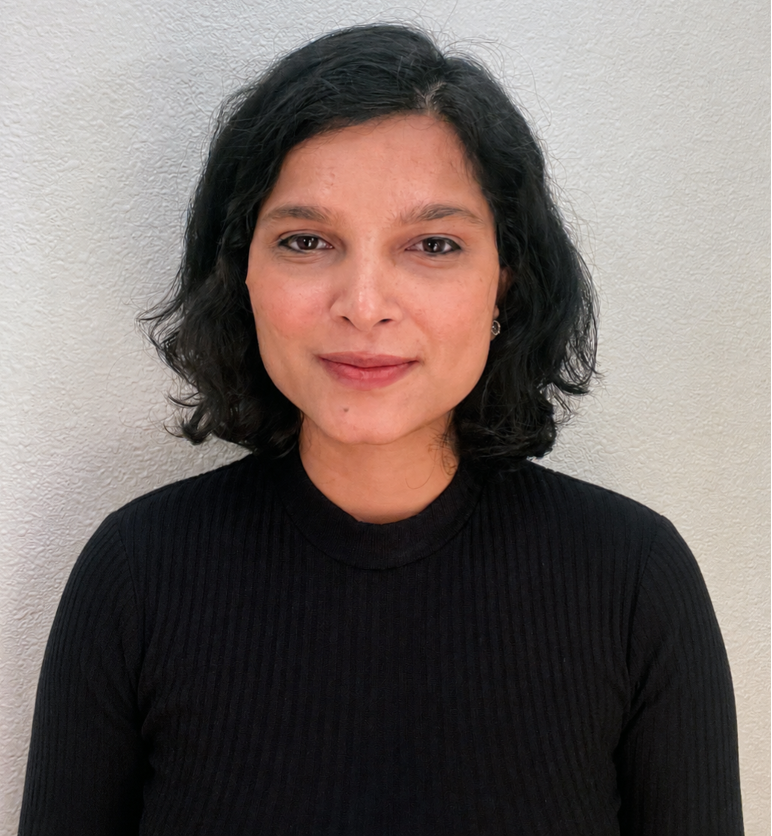}}]{Pooja Algikar}is a Postdoctoral Researcher at Argonne National Laboratory. She received the Ph.D. from Virginia Tech in Electrical and Computer Engineering in 2025. Her research focuses on machine learning for energy systems, graph representation learning, and uncertainty quantification. Her work also spans Bayesian learning, Gaussian processes, robust estimation, game-theoretic security, and Koopman operator methods.
\end{IEEEbiography}

\begin{IEEEbiography}[{\includegraphics[width=1in,height=1.25in,clip,keepaspectratio]{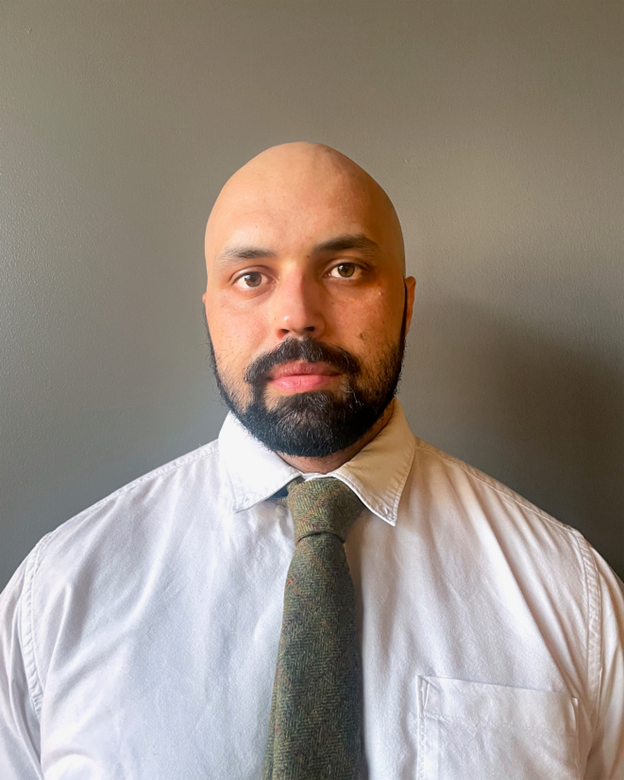}}]{Stefano Fenu}is an associate computational scientist at Argonne National Laboratory. He received his Ph.D. in Computer Science from Georgia Tech in 2023. His primary research interests are in machine learning, search, dimensionality reduction, and optimization, particularly with applications to power systems
\end{IEEEbiography}

\begin{IEEEbiography}[{\includegraphics[width=1in,height=1.25in,clip,keepaspectratio]{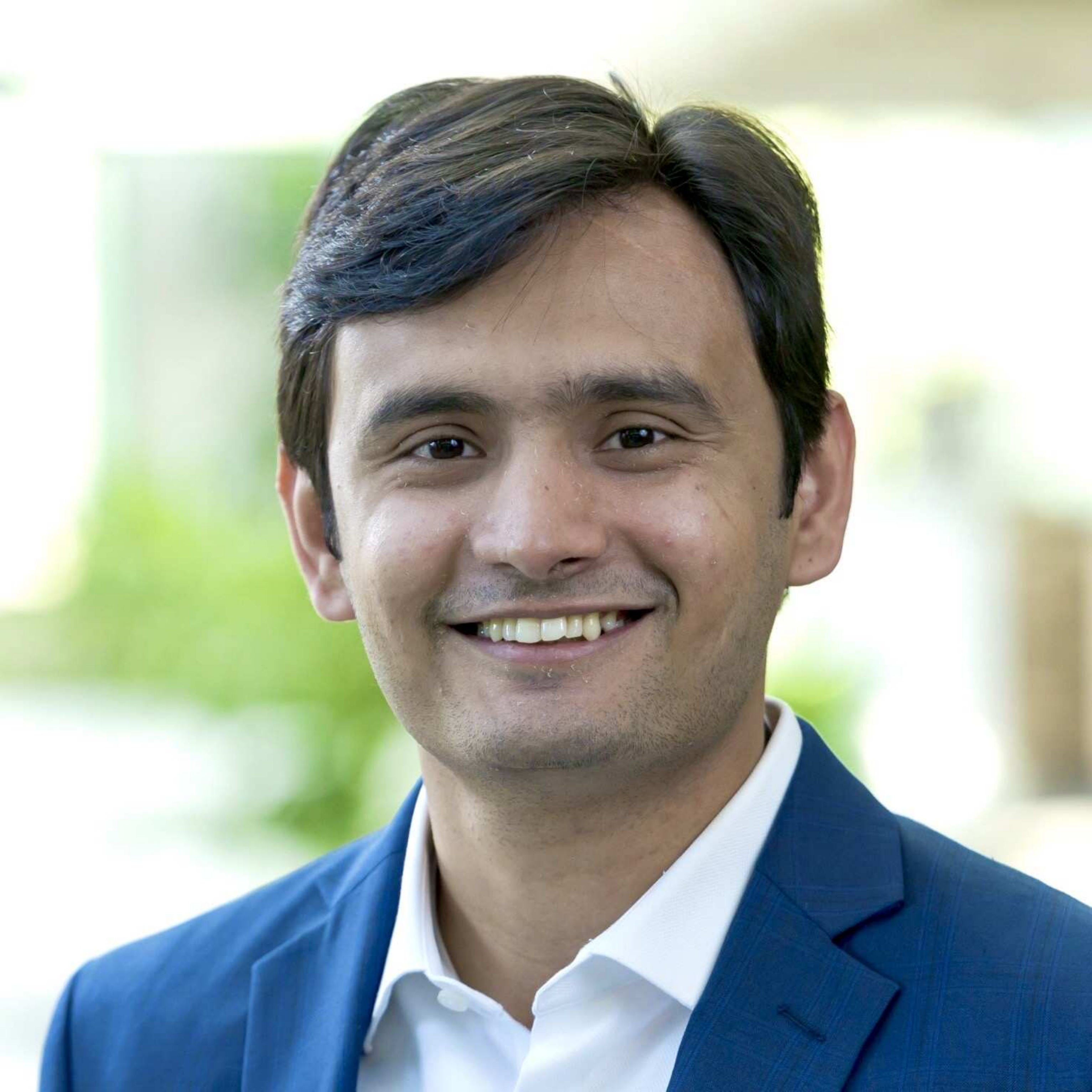}}]{Kumarsinh Jhala} (M'18) is a Principal Energy Systems Scientist in the Center for Energy, Environmental, and Economic Systems Analysis in the Energy Systems and Infrastructure Analysis Division at Argonne National Laboratory, Lemont, IL, USA. \\
Dr. Jhala received the Ph.D. and the M.S. degree in Electrical and Computer Engineering from Kansas State University, Manhattan, KS, USA in 2018 and 2015, respectively. He has performed extensive research on distribution system modeling and analysis. His area of expertise spans domains of transactive energy systems, cyber-physical security, machine learning, statistical  signal processing, game theory, stochastic modeling, optimization and control theory. He is an associate editor for the IEEE transactions on Smart Grid.
\end{IEEEbiography}

\begin{IEEEbiography}[{\includegraphics[width=1in,height=1.25in,clip,keepaspectratio]{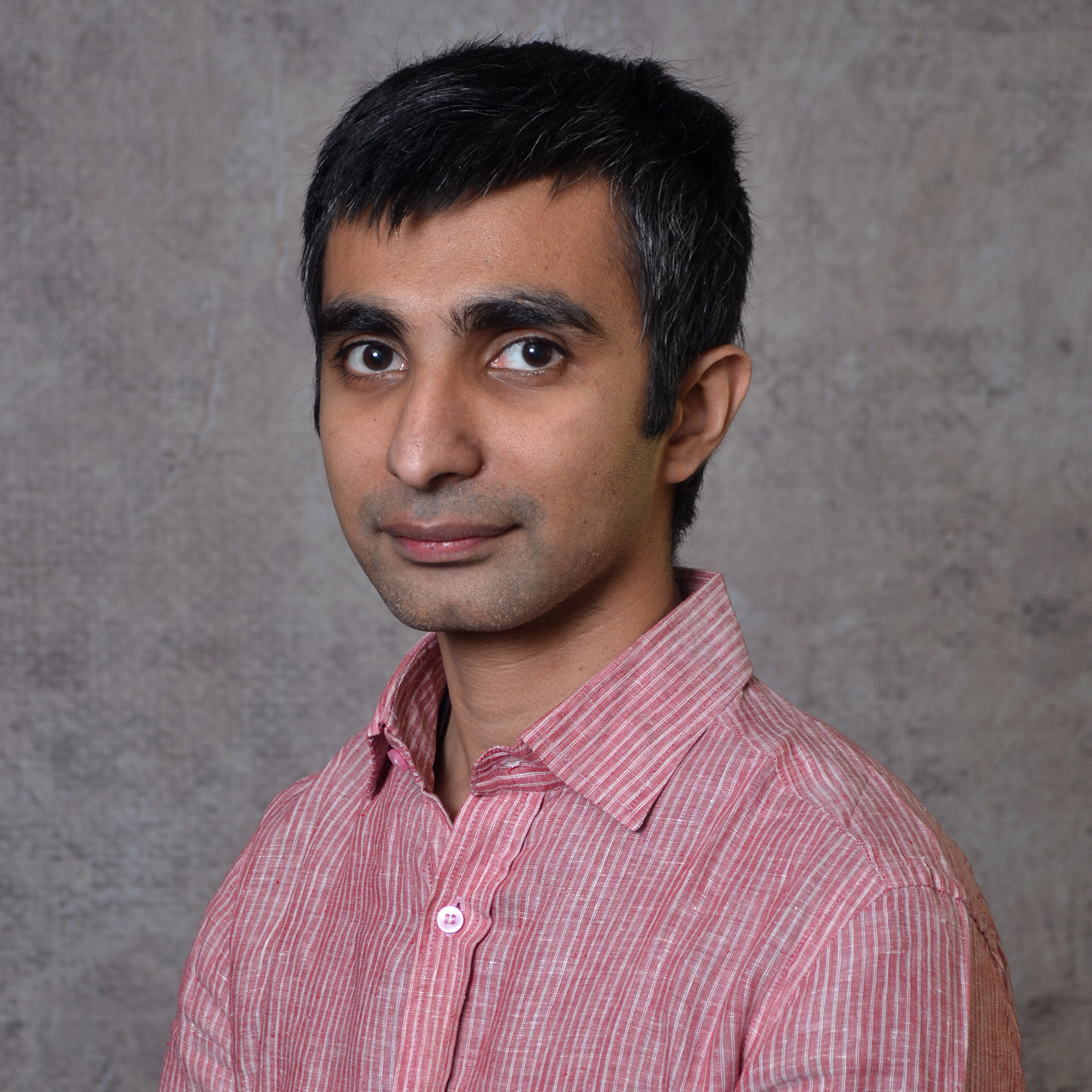}}]{Siby Plathottam} (M'15) is an Energy Systems Scientist in the Energy Systems and Infrastructure Analysis Division at Argonne National Laboratory, Lemont, IL, USA. \\
Dr. Plathottam received the Ph.D in Electrical Engineering from University of North Dakota, Grand Forks, ND, USA in 2017. He has worked on projects for the U.S. Department of Energy Office of Electricity, Buildings Technology Office, and Advanced Manufacturing office and has expertise in distributed energy resource modeling, applied machine learning, and open-source simulation software development.
\end{IEEEbiography}

\end{document}